\documentclass{article}

\PassOptionsToPackage{table}{xcolor}
\usepackage{report}

\usepackage[utf8]{inputenc} 
\usepackage[T1]{fontenc}    
\usepackage{hyperref}       
\usepackage{url}            
\usepackage{booktabs}       
\usepackage{amsfonts}       
\usepackage{nicefrac}       
\usepackage{microtype}      
\usepackage{xcolor}         
\usepackage{graphicx}

\RequirePackage{tikz}
\usetikzlibrary{arrows.meta, shapes.geometric, positioning, 
                matrix, decorations.pathreplacing, 
                patterns, calc, backgrounds, fit, 
                shadows, intersections}

\RequirePackage{graphicx}
\RequirePackage{subcaption}
\RequirePackage{caption}
\RequirePackage{float}
\RequirePackage{placeins}

\RequirePackage{xcolor}
\RequirePackage{colortbl}
\RequirePackage{multirow}
\RequirePackage{booktabs}

\RequirePackage{amsmath, amssymb}
\RequirePackage{bm}
\RequirePackage{pifont}

\usepackage{amsmath}      
\usepackage{amssymb}      
\usepackage{amsfonts}     
\usepackage{amsthm}       

\newtheorem{theorem}{Theorem}[section]          
\newtheorem{lemma}[theorem]{Lemma}              
\newtheorem{assumption}[theorem]{Assumption}

\usepackage{algorithm}
\usepackage{algorithmic}

\usepackage[most]{tcolorbox}
\tcbuselibrary{listings,breakable,skins}
\tcbset{
  promptbox/.style={
    enhanced,
    breakable,
    colback=gray!3,
    colframe=black!35,
    boxrule=0.6pt,
    arc=2pt,
    left=6pt,right=6pt,top=6pt,bottom=6pt,
    listing only,
    listing options={
      basicstyle=\ttfamily\footnotesize,
      breaklines=true,
      breakatwhitespace=false,
      columns=fullflexible,
      keepspaces=true,
      showstringspaces=false
    }
  }
}

\definecolor{njuPurple}{RGB}{220,205,230}
\definecolor{njuPurpleLight}{RGB}{250,245,252}

\newtcolorbox{abstractbox}{
    colback=njuPurpleLight,
    colframe=njuPurple,
    boxrule=1pt,
    arc=4mm,
    left=8pt,
    right=8pt,
    top=8pt,
    bottom=8pt,
    opacityback=0.95,
    breakable
}

\usepackage{xcolor}
\usepackage{colortbl}

\usepackage{wrapfig}

\newcommand{\best}[1]{\colorbox{blue!30}{\strut \textbf{#1}}}
\newcommand{\second}[1]{\colorbox{blue!18}{\strut \textbf{#1}}}
\newcommand{\third}[1]{\colorbox{blue!10}{\strut \textbf{#1}}}
\title{InduceKV: Fixed-Footprint Continual Adaptation of Multimodal LLMs via Inducing KV Memories}

\author{%
Qianyu Chen$^{1}$ \quad
Ziteng Feng$^{1,2}$ \quad
Canran Xiao$^{3}$ \quad
Runxuan Tang$^{1}$\\[4pt]
$^{1}$Nanyang Technological University\\
$^{2}$University of Science and Technology of China\\
$^{3}$Shenzhen Campus of Sun Yat-sen University
}

\begin{document}

\maketitle

\begin{abstract}
\begin{abstractbox}
Multimodal large language models must adapt to evolving tasks and domains, yet continual improvement under bounded deployment footprint remains difficult because repeated parameter updates or growing replay stores can accumulate adaptation state over time.
We study fixed-footprint continual adaptation: the deployed adaptation state is kept under a fixed memory budget, while the backbone model is left unchanged and task-specific updates are externalized.
We propose \textsc{InduceKV}, a retrieval-based method that stores each selected training prefix as an attention-ready memory entry, consisting of a frozen retrieval key and compact layerwise key--value (KV) payloads that can be appended to the model's self-attention cache.
Under a strict memory budget, \textsc{InduceKV} constructs a compact inducing set through bilevel selection: a lightweight calibration is fit for retrieval, while the selected memory balances current-task likelihood, anchor-based retention, and coverage in the frozen retrieval space.
Across task-incremental instruction tuning, continual VQA, domain-incremental adaptation, and lifelong multimodal instruction tuning, \textsc{InduceKV} consistently improves over PEFT, MoE, replay, and prompt-retrieval baselines under matched memory budgets.
We further report backbone-matched, stage-1 CoIN, compute-matched, and scalability diagnostics, showing that the gains are not due to a stronger backbone, replay alone, or an unbounded candidate pool.
\end{abstractbox}
\end{abstract}

\section{Introduction}
Multimodal large language models (MLLMs)~\cite{yin2024survey} are becoming a common interface for visual understanding and instruction-following in real applications, from continual VQA~\cite{jian2024large,jia2025vqa2} to domain-specialized assistants~\cite{bian2025mira} and tool-augmented reasoning~\cite{hu2025agents}.
In these deployments, new tasks, domains, and datasets arrive over time, yet repeatedly retraining or jointly re-tuning on all historical data is often impractical due to compute, privacy, and product iteration constraints~\cite{huo2025continue}.
The central challenge is therefore to continually adapt an MLLM while preserving prior competencies under a fixed deployed adaptation footprint.
In this work, ``fixed footprint'' refers to a bounded external adaptation state after each task.

Despite rapid progress, existing continual adaptation pipelines for VLMs/MLLMs expose a persistent methodological gap.
A dominant line of work updates model parameters---typically through PEFT modules, expert routing, or architecture expansion---to trade off stability and plasticity in sequential instruction tuning and continual VQA \cite{Hidellava25,clmoe25,SMoE25,DMoLE25,BranchLoRA25,SMoLoR,mode25}.
While these designs are effective, they inherently rewrite parts of the model over time, which can complicate budget control (e.g., growing expert/router state) and can entangle new-task adaptations with the pretrained cross-modal interface that underpins broad generalization.
A second line reduces forgetting via replay-like mechanisms and synthetic or selective rehearsal, including dynamic data selection for lifelong instruction tuning and memory-efficient replay/distillation strategies \cite{Adapt25,quad25,gift25}.
However, these approaches still couple continual improvement to training-time data management and can face tension between (i) preserving diverse historical behaviors and (ii) fitting within a fixed, small memory footprint, especially when the candidate pool becomes large and redundant.
Finally, continual learning on CLIP-style VLMs has highlighted the importance of retaining pretrained, zero-shot capabilities while learning new domains \cite{cclip25}, suggesting that continual adaptation must be careful about \emph{where} and \emph{how} new information is injected into the inference pathway.

This paper asks whether scalable continual adaptation is possible by keeping the multimodal backbone fixed and updating a compact external memory that affects generation through native attention.
From a retrieval perspective, the key challenge is maintaining a budgeted, non-redundant inducing set that improves current likelihood while preserving historical competence.
We therefore cast continual adaptation as budgeted online inducing-set construction with lightweight retrieval calibration, injecting task-specific evidence without repeated backbone updates, as shown in Fig.\ref{fig.teaser}.

\begin{figure}[H]
    \centering
    \includegraphics[width=0.72\linewidth]{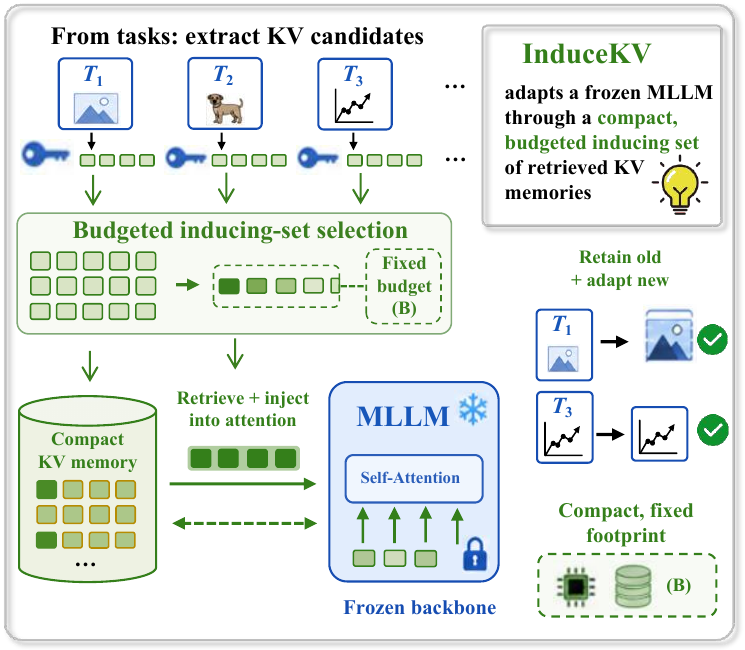}
    \caption{\textbf{\textsc{InduceKV} for budgeted continual MLLM adaptation.}}
    \label{fig.teaser}
\end{figure}

Our contributions are as follows:
(\textbf{\textit{i}}) We reframe continual MLLM adaptation as budgeted online inducing-set selection for retrieval-based memory, making the stability--plasticity tension explicit as a memory allocation problem rather than an ever-growing parameter update process.
(\textbf{\textit{ii}}) We introduce a retrieval-driven adaptation mechanism that stores task increments as attention-compatible external KV memories and uses a lightweight calibration interface to control retrieval strength, while explicitly accounting for its prefill and KV-injection overhead.
(\textbf{\textit{iii}}) Our approach achieves consistent state-of-the-art results across diverse continual instruction-tuning and continual VQA settings.

\section{Related Work}
\label{sec:related}

\textbf{Continual adaptation of instruction-following MLLMs.}
Most continual multimodal instruction tuning and VQA methods adapt models by \emph{updating parameters} via PEFT and modularization.
Representative directions include dynamic data selection/pruning for lifelong instruction tuning \cite{Adapt25}, PEFT-based isolation and routing such as hierarchical layer decoupling \cite{Hidellava25}, mixtures of LoRA experts \cite{DMoLE25,SMoLoR,BranchLoRA25}, and projector-centric adaptation \cite{MVP25}.
For continual VQA, MoE routing and replay/distillation (e.g., dual-router momentum MoE and question-only replay with attention distillation) achieve strong stability--plasticity trade-offs \cite{clmoe25,quad25}, while domain-incremental settings often use domain-specific experts and routing \cite{SMoE25}.
However, these approaches still require continual gradient updates, often expand expert/router state or depend on rehearsal, and tie retention to parameter-space constraints that can limit plasticity under tight budgets.
In contrast, \textsc{InduceKV} freezes the backbone and influences generation by injecting retrieved, attention-ready KV payloads into the masked self-attention cache, learning only a tiny per-stream calibration.

\textbf{Continual learning in CLIP-style VLMs.}
Recent CLIP-style continual learning methods combine adapters/regularization with consolidation objectives and new benchmarks for domain/class shifts \cite{cclip25,DNS25}, and mitigate forgetting via synthetic replay/distillation \cite{gift25}, structured adapters or memory units \cite{lada25}, modality-gap constraints \cite{mgclip25}, external knowledge injection \cite{engine25}, and feature-geometry preserving regularization \cite{proxyfda25}.
However, extending these ideas to autoregressive MLLMs is nontrivial: generation is governed by layerwise attention and long-range cache interactions, so storing past data alone may not yield controllable influence at decoding time.
We address this gap by storing layerwise KV payloads from the frozen generator and injecting them directly into self-attention.

\textbf{Retrieval-augmented memory and budgeted subset selection.}
Retrieval-augmented LMs use non-parametric memory via passage retrieval \cite{rag20,retro22} or kNN retrieval over hidden states/KV pairs \cite{knnlm20,memtransformer22}.
In parallel, coreset and diversity selection study coverage under budgets, often via log-det/DPP-style objectives \cite{kulesza12dpp,wang24dpp,staff25}.
Yet most retrieval systems assume a static corpus and optimize query-time evidence, not continual adaptation under a fixed footprint, and subset selection rarely co-designs inference-time retrieval calibration with an attention-aligned injection pathway.
\textsc{InduceKV} fills this gap by formulating memory updates as online inducing-set selection: a lightweight inner calibration adapts retrieval/gating, while the outer loop selects a budgeted subset that balances current fit, anchor-based retention, and log-det spectral coverage.

\begin{figure*}[t]
	\centering
	\includegraphics[width=0.75\linewidth]{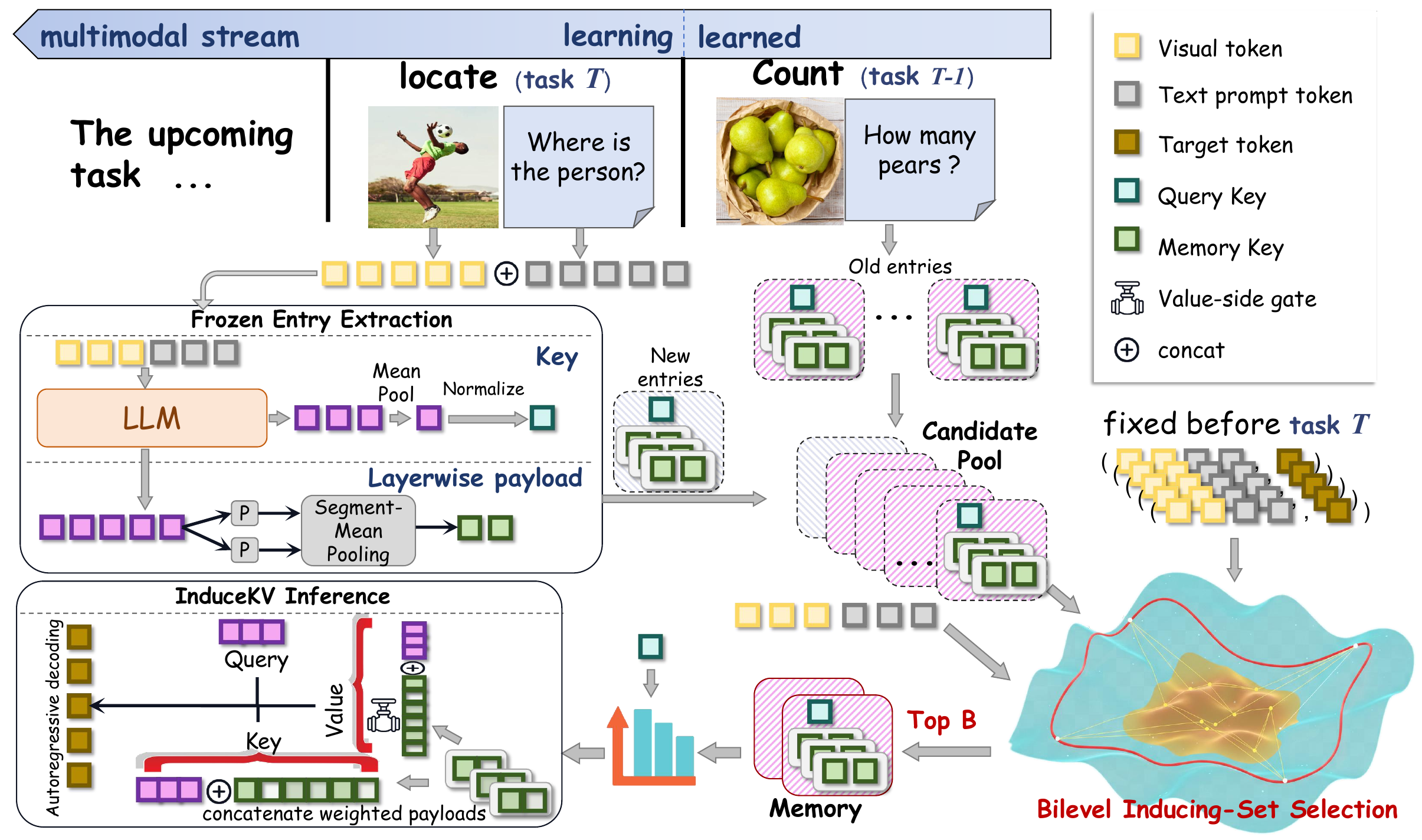}
\caption{\textbf{\textsc{InduceKV} pipeline for continual adaptation.}
For each incoming task $t$, the frozen MLLM extracts from each prefix $x$ a unit-norm retrieval key $r(x)$ and compressed layerwise KV payloads $\{(\bar K^\ell(x),\bar V^\ell(x))\}_{\ell=1}^L$, forming new entries that are merged with the previous memory into a candidate pool $\mathcal{U}_t$.
Under a fixed budget $B$, a bilevel optimizer constructs a compact inducing set: the inner level fits a minimal retrieval calibration $\phi^\star(w)$ on current-task data, while the outer level updates selection weights $w$ to jointly optimize current-task likelihood, retention on historical anchors $\mathcal{A}_{<t}$, and a redundancy-aware spectral coverage regularizer in the frozen retrieval space.
The Top-$B$ entries are committed as the external KV memory $\mathcal{M}_t$.
At inference, the query key retrieves weighted entries, whose KV payloads are concatenated and injected as additional cache tokens into masked self-attention, enabling task adaptation without updating $\theta$.}
	\label{fig:pipeline}
\end{figure*}

\section{Method}
\label{sec:method}
\textsc{InduceKV} performs continual adaptation of a multimodal large language model (MLLM) by freezing all backbone parameters and externalizing task-specific increments into an attention-compatible key--value (KV) memory. Each memory entry stores (i) a frozen-model retrieval key in a fixed representation space and (ii) compact, layerwise KV payloads that can be injected into the Transformer self-attention cache. Under a strict memory budget, we construct a compact inducing set via bilevel optimization: an inner problem fits a minimal retrieval calibration parameter on the current task, while an outer problem selects a budgeted subset that jointly optimizes current-task likelihood, retention on historical anchors, and a redundancy-aware spectral coverage regularizer in the frozen retrieval space. The pipeline is shown in Fig.~\ref{fig:pipeline}.

\subsection{Problem Setup}
We consider tasks $t=1,\dots,T$ arriving sequentially. At task $t$, we observe $\mathcal{D}_t=\{(x_n,y_n)\}_{n=1}^{N_t}$, where $x$ is the \emph{conditioning prefix} (multimodal input and prompt) and $y$ is the target output sequence. We freeze an MLLM with parameters $\theta$ and perform no gradient updates to $\theta$ at any time.
We maintain a small historical anchor set $\mathcal{A}_{<t}$ (fixed before training task $t$) used to measure retention.

\textbf{Autoregressive objective.}
For a pair $(x,y)$, the negative log-likelihood is
\begin{equation}
\mathcal{L}_{\mathrm{NLL}}(\theta;\,x,y)
= -\sum_{m=1}^{|y|}\log p_\theta\!\left(y_m \mid y_{<m}, x\right).
\label{eq:nll_rewrite}
\end{equation}
In \eqref{eq:nll_rewrite}, $p_\theta(\cdot)$ is the frozen model distribution, and $x$ excludes all answer tokens to prevent label leakage in memory construction.

\textbf{Multi-head attention convention.}
We write self-attention in per-head form and omit the head index for clarity. Let $d_h$ denote the per-head key dimension; the same construction is applied independently to each head and each layer.

\subsection{Frozen Key and KV-Payload Extraction}
The goal is to store \emph{attention-ready} information from the frozen model: a retrieval key for indexing in a fixed representation space, and compact KV payloads that can be appended to the self-attention cache without changing $\theta$. This avoids parameter-efficient finetuning that may perturb cross-modal interactions.

Given a prefix input $x$, let $H^{(L)}(x)\in\mathbb{R}^{n_p(x)\times d}$ be the final-layer hidden states corresponding to the prefix tokens (length $n_p(x)$, hidden size $d$). We define the pooled prefix representation and normalized retrieval key as
\begin{equation}
\bar h(x) = \frac{1}{n_p(x)}\sum_{j=1}^{n_p(x)}H^{(L)}_{j}(x)\in\mathbb{R}^{d}, \qquad
r(x) = \frac{\bar h(x)}{\|\bar h(x)\|_2}\in\mathbb{R}^{d}.
\label{eq:key_rewrite}
\end{equation}
In \eqref{eq:key_rewrite}, $H^{(L)}_{j}(x)$ is the $j$-th prefix token vector, and $r(x)$ is unit-norm, enabling cosine-similarity retrieval.

\textbf{Layerwise KV payload.}
For each Transformer layer $\ell\in\{1,\dots,L\}$, let $H^{(\ell-1)}(x)\in\mathbb{R}^{n_p(x)\times d}$ be the layer input over prefix tokens. Using the frozen per-head projections $W_K^\ell\in\mathbb{R}^{d\times d_h}$ and $W_V^\ell\in\mathbb{R}^{d\times d_h}$, we compute prefix keys/values and compress them to a fixed payload length $m$:
\begin{equation}
\begin{aligned}
K^\ell(x) &= H^{(\ell-1)}(x)W_K^\ell \in \mathbb{R}^{n_p(x)\times d_h}, \qquad &
V^\ell(x) &= H^{(\ell-1)}(x)W_V^\ell \in \mathbb{R}^{n_p(x)\times d_h}, \\
\bar K^\ell(x) &= P_m(x)K^\ell(x), \qquad &
\bar V^\ell(x) &= P_m(x)V^\ell(x).
\end{aligned}
\label{eq:payload_rewrite}
\end{equation}
In \eqref{eq:payload_rewrite}, $P_m(x)\in\mathbb{R}^{m\times n_p(x)}$ is a fixed pooling matrix that partitions the prefix indices $\{1,\dots,n_p(x)\}$ into $m$ consecutive segments of equal size (the last segment absorbs any remainder) and averages within each segment; therefore $\bar K^\ell(x),\bar V^\ell(x)\in\mathbb{R}^{m\times d_h}$ have a task-independent token length $m$.

\textbf{Memory entry.}
We represent the extracted entry from $x$ as
$
e(x)=\big(r(x),\{(\bar K^\ell(x),\bar V^\ell(x))\}_{\ell=1}^L\big)
$,
and store only such entries in the external memory.

\subsection{Retrieval-Weighted KV Induction and Masked Attention Injection}
At inference, we want retrieved knowledge to influence generation through the same pathway as standard Transformer caches. \textsc{InduceKV} therefore injects memory KV payloads as additional cache tokens into each layer's masked self-attention.
Let $\mathcal{M}=\{e_i\}_{i=1}^{|\mathcal{M}|}$ be the current memory, with stored keys $r_i\in\mathbb{R}^{d}$.
We learn a minimal calibration parameter $\phi=\{\phi_\tau,\phi_1,\dots,\phi_L\}$, where the retrieval temperature is $\tau=\mathrm{softplus}(\phi_\tau)>0$ and the per-layer value gate is $\lambda_\ell=\sigma(\phi_\ell)\in(0,1)$.
Given a prefix $x$, we compute $r(x)$ by \eqref{eq:key_rewrite} and define the retrieval weight for entry $i$ as
\begin{equation}
\alpha_i(x;\phi)
=
\frac{\exp\!\big(\langle r(x),r_i\rangle/\tau\big)}
{\sum_{j=1}^{|\mathcal{M}|}\exp\!\big(\langle r(x),r_j\rangle/\tau\big)}.
\label{eq:alpha_rewrite}
\end{equation}
where $\langle r(x),r_i\rangle$ is cosine similarity since both keys are unit-norm, and $\sum_i \alpha_i(x;\phi)=1$.

For layer $\ell$, we assemble memory KV tokens by concatenating payload blocks weighted by $\sqrt{\alpha_i(x;\phi)}$:
$
K_{\mathrm{mem}}^\ell(x)=\mathrm{Concat}\big(\sqrt{\alpha_i(x;\phi)}\,\bar K_i^\ell\big)_{i=1}^{|\mathcal{M}|}
\in\mathbb{R}^{(|\mathcal{M}|m)\times d_h}
$
and
$
V_{\mathrm{mem}}^\ell(x)=\mathrm{Concat}\big(\sqrt{\alpha_i(x;\phi)}\,\bar V_i^\ell\big)_{i=1}^{|\mathcal{M}|}
\in\mathbb{R}^{(|\mathcal{M}|m)\times d_h}
$,
where $\mathrm{Concat}(\cdot)$ concatenates along the token dimension.
Let $Q^\ell(x)\in\mathbb{R}^{n_p(x)\times d_h}$, $K_{\mathrm{self}}^\ell(x)\in\mathbb{R}^{n_p(x)\times d_h}$, and $V_{\mathrm{self}}^\ell(x)\in\mathbb{R}^{n_p(x)\times d_h}$ be the standard frozen per-head tensors produced at layer $\ell$ for the prefix tokens.
We inject memory KV as additional cache tokens and compute masked self-attention as
\begin{equation}
\begin{aligned}
A^\ell(x) &= \frac{
    Q^\ell(x)\big[\,K_{\mathrm{self}}^\ell(x);\ K_{\mathrm{mem}}^\ell(x)\,\big]^\top
    + M^\ell(x)
}{
    \sqrt{d_h}
} \\
\mathrm{Attn}^\ell(x) &= \mathrm{softmax}\!\left(A^\ell(x)\right)
    \big[\,V_{\mathrm{self}}^\ell(x);\ \lambda_\ell V_{\mathrm{mem}}^\ell(x)\,\big].
\end{aligned}
\label{eq:inject_rewrite}
\end{equation}
In \eqref{eq:inject_rewrite}, $[\cdot;\cdot]$ denotes token-wise concatenation and $M^\ell(x)$ is the extended causal mask that permits attention to all memory tokens while enforcing standard causal visibility among prefix tokens.
The gate $\lambda_\ell$ scales only memory values, providing an interpretable per-layer control of memory strength.

\paragraph{Online inference path and cost.}
At test time, \textsc{InduceKV} uses a two-stage inference path.
First, it runs a prefix-only forward pass with the frozen backbone to compute the query retrieval key $r(x)$.
Second, after retrieving entries from the fixed memory $\mathcal{M}$, it performs the generation forward pass with the selected precomputed KV payloads appended to the self-attention cache.
The stored memory payloads $\{(\bar K_i^\ell,\bar V_i^\ell)\}$ are extracted offline at task-update time and are never recomputed for each test query.
Thus, the online overhead is not an additional payload-extraction pass, but an extra prefix pass for retrieval-key computation plus larger attention matrices during prefill.
We therefore position \textsc{InduceKV} as a fixed-footprint retrieval-based continual adaptation method, not as a zero-overhead replacement for replay-free PEFT.
Its practical advantage is most direct against prompt-level retrieval, where retrieved exemplars must be processed through embeddings, projections, attention, and FFN layers at inference.

\subsection{Bilevel Inducing-Set Selection Under a Fixed Budget}
Storing all extracted entries is redundant under a fixed budget. \textsc{InduceKV} therefore builds a compact inducing set that fits the current task via minimal calibration while preserving performance on historical anchors, and enforces spectral coverage in the frozen retrieval space to avoid near-duplicate keys.

\textbf{Candidate pool and selection variables.}
At task $t$, we build a candidate pool $\mathcal{U}_t=\mathcal{C}_t\cup\mathcal{M}_{t-1}$, where $\mathcal{C}_t$ contains entries extracted from prefixes in $\mathcal{D}_t$ and $\mathcal{M}_{t-1}$ is the previous memory.
Let $N=|\mathcal{U}_t|$ and index candidates by $i\in\{1,\dots,N\}$.
We introduce continuous selection weights $w\in\mathbb{R}_{\ge 0}^{N}$ constrained by
$
\sum_{i=1}^{N} w_i=B
$,
where $B$ is the fixed entry budget.
We interpret $\pi_i(w)=w_i/B$ as a fractional inclusion mass with $\sum_i \pi_i(w)=1$.

\textbf{Retrieval during selection.}
While optimizing $w$, we define retrieval weights over candidates by incorporating $\pi_i(w)$ as a multiplicative prior:
$
\alpha_i(x;w,\phi)=
\frac{\pi_i(w)\exp(\langle r(x),r_i\rangle/\tau)}{\sum_{j=1}^{N}\pi_j(w)\exp(\langle r(x),r_j\rangle/\tau)}
$.
This yields a differentiable dependence of induced KV on $w$.

\textbf{Inner problem (fit minimal calibration on current task).}
Given $w$, the inner level fits $\phi$ by minimizing current-task likelihood with $\ell_2$ regularization:
\begin{equation}
\begin{aligned}
\phi^\star(w) = \arg\min_{\phi}\;
\frac{1}{|\mathcal{D}_t|}\sum_{(x,y)\in\mathcal{D}_t}
&\mathcal{L}_{\mathrm{NLL}}\big(\theta;\,x,y \mid \mathcal{U}_t,w,\phi\big) \\
&+ \eta\|\phi\|_2^2.
\end{aligned}
\label{eq:inner_rewrite}
\end{equation}
In \eqref{eq:inner_rewrite}, the conditional notation $\mathcal{L}_{\mathrm{NLL}}(\cdot\mid \mathcal{U}_t,w,\phi)$ means that the model uses candidate-based retrieval weights $\alpha_i(x;w,\phi)$ and injects the induced KV via \eqref{eq:inject_rewrite}; $\eta>0$ is fixed.

\paragraph{Outer problem (select inducing set with retention and coverage).}
The outer level selects $w$ by balancing current fit, anchor retention, and spectral coverage:
\begin{equation}
w^\star = \arg\min_{w \in \mathcal{W}}\; J(w),
\label{eq:outer_rewrite}
\end{equation}
where $\mathcal{W} = \{w \in \mathbb{R}^N_{\ge 0} : \sum_i w_i = B\}$, and the objective is
\begin{align}
J(w) &= \underbrace{\mathcal{L}_{\mathrm{cur}}\big(w,\phi^\star(w)\big)}_{\text{current fit}} 
      + \beta\underbrace{\mathcal{L}_{\mathrm{anc}}\big(w,\phi^\star(w)\big)}_{\text{retention}} 
      + \gamma\underbrace{\Omega_{\mathrm{spec}}(w)}_{\text{coverage}}.
\end{align}

where $\beta\ge 0$ and $\gamma\ge 0$ are fixed coefficients,
$
\mathcal{L}_{\mathrm{cur}}(w,\phi)=\frac{1}{|\mathcal{D}_t|}\sum_{(x,y)\in\mathcal{D}_t}\mathcal{L}_{\mathrm{NLL}}(\theta;\,x,y\mid \mathcal{U}_t,w,\phi)
$,
and
$
\mathcal{L}_{\mathrm{anc}}(w,\phi)=\frac{1}{|\mathcal{A}_{<t}|}\sum_{(x,y)\in\mathcal{A}_{<t}}\mathcal{L}_{\mathrm{NLL}}(\theta;\,x,y\mid \mathcal{U}_t,w,\phi)
$.

\textbf{Spectral coverage regularizer.}
Let $r_i\in\mathbb{R}^{d}$ be the unit-norm retrieval key of candidate $i$.
We fix a projection matrix $P\in\mathbb{R}^{d\times d'}$ with orthonormal columns ($P^\top P=I_{d'}$) and set $d'=256$.
Define projected keys $z_i=P^\top r_i\in\mathbb{R}^{d'}$ and the weighted covariance
$
C(w)=\sum_{i=1}^{N}\pi_i(w)\,z_i z_i^\top\in\mathbb{R}^{d'\times d'}.
$
We use a log-determinant coverage penalty
\begin{equation}
\Omega_{\mathrm{spec}}(w)=
-\log\det\!\big(C(w)+\epsilon I_{d'}\big),
\label{eq:spec_rewrite}
\end{equation}
where $\epsilon>0$ is a fixed stabilizer and $I_{d'}$ is the $d'\times d'$ identity. Minimizing \eqref{eq:spec_rewrite} discourages redundant keys by penalizing low-rank, spectrally concentrated $C(w)$ in the frozen retrieval space.

We solve \eqref{eq:inner_rewrite}--\eqref{eq:outer_rewrite} by unrolling a fixed number of inner steps for $\phi$ and applying projected gradient updates to $w$ onto $\{w\ge 0:\sum_i w_i=B\}$.
After convergence, we take $\mathcal{S}_t=\mathrm{TopB}(w^\star)$ and set $\mathcal{M}_t=\{e_i\}_{i\in\mathcal{S}_t}$.
For later tasks, $\mathcal{M}_t$ is fixed, and we optimize only the new task's $(w,\phi)$.

\begin{algorithm}[t]
\caption{\textsc{InduceKV} update at task $t$}
\label{alg:inducekv_rewrite}
\begin{algorithmic}[1]
\STATE \textbf{Input:} $\mathcal{D}_t$, $\mathcal{A}_{<t}$, $\mathcal{M}_{t-1}$, budget $B$
\STATE Build candidates $\mathcal{U}_t=\mathcal{C}_t\cup\mathcal{M}_{t-1}$; extract entries via \eqref{eq:key_rewrite}--\eqref{eq:payload_rewrite}
\STATE Initialize $w\ge 0$ with $\sum_i w_i=B$ and initialize $\phi$
\FOR{$\mathrm{iter}=1$ to $I$}
    \FOR{$j=1$ to $J$}
        \STATE $\phi \leftarrow \phi - \alpha_\phi \nabla_\phi\Big(\frac{1}{|\mathcal{D}_t|}\sum_{(x,y)\in\mathcal{D}_t}\mathcal{L}_{\mathrm{NLL}}(\theta;\,x,y\mid \mathcal{U}_t,w,\phi)+\eta\|\phi\|_2^2\Big)$
    \ENDFOR
    \STATE $w \leftarrow w - \alpha_w \nabla_w\Big(\mathcal{L}_{\mathrm{cur}}(w,\phi)+\beta\mathcal{L}_{\mathrm{anc}}(w,\phi)+\gamma\Omega_{\mathrm{spec}}(w)\Big)$
    \STATE Project $w$ onto $\{w\ge 0:\sum_i w_i=B\}$
\ENDFOR
\STATE $\mathcal{S}_t\leftarrow \mathrm{TopB}(w)$,\quad $\mathcal{M}_t\leftarrow\{e_i\}_{i\in\mathcal{S}_t}$,\quad $\phi_t\leftarrow\phi$
\STATE \textbf{Output:} memory $\mathcal{M}_t$, calibration $\phi_t$
\end{algorithmic}
\end{algorithm}

\section{Theory: Budgeted Continual Learning as Online Inducing-Set Selection}
\label{sec:theory_oco}

We give a compact theoretical view of \textsc{InduceKV}: after the inner calibration is solved, the outer memory update becomes an online budgeted subset-selection problem over relaxed inducing-set weights.
Let \(f_t(w)\) denote the inner-solved outer surrogate at task \(t\), defined on the budget simplex \(\Delta_B^N=\{w\ge 0:\sum_i w_i=B\}\), and let
\(D=\max_{u,v\in\Delta_B^N}\|u-v\|_2=B\sqrt{2}\).
The outer update is projected gradient descent,
\(w_{t+1}=\Pi_{\Delta_B^N}(w_t-\eta\nabla f_t(w_t))\),
followed by the \(\mathrm{TopB}\) discretization used in Alg.~\ref{alg:inducekv_rewrite}.
Full definitions, assumptions, and proofs are provided in Appendix~\ref{app:theory_oco}.

\paragraph{Budgeted selection has sublinear regret.}
The following result shows that the relaxed memory selector competes with the best fixed inducing-set allocation in hindsight under the same budget.

\begin{theorem}[Static regret under fixed budget]
\label{thm:static_regret}
Assume each inner-solved surrogate \(f_t\) is convex and has bounded gradients \(\|\nabla f_t(w)\|_2\le G\) on \(\Delta_B^N\).
Then for any comparator \(u\in\Delta_B^N\), the projected outer updates satisfy
\begin{align}
\mathrm{Reg}_T(u)
&\triangleq
\sum_{t=1}^{T}\big(f_t(w_t)-f_t(u)\big)
\le
\frac{D^2}{2\eta}+\frac{\eta G^2T}{2},
\label{eq:static_regret_bound}\\
\mathrm{Reg}_T(u)
&\le
DG\sqrt{T}
=
B\sqrt{2}\,G\sqrt{T}
\qquad
\text{when } \eta=\frac{D}{G\sqrt{T}}.
\label{eq:static_regret_simplified}
\end{align}
\end{theorem}

\paragraph{The bound adapts to nonstationary task streams.}
If the best inducing-set allocation changes over time, the regret increases only with the path-length variation of the changing comparator.

\begin{theorem}[Dynamic regret with task variation]
\label{thm:dynamic_regret}
Let \(\{u_t\}_{t=1}^{T}\subset\Delta_B^N\) be any time-varying comparator sequence and
\(V_T=\sum_{t=1}^{T-1}\|u_{t+1}-u_t\|_2\).
Under the same bounded-gradient condition,
\begin{align}
\mathrm{DReg}_T(\{u_t\})
&\triangleq
\sum_{t=1}^{T}\big(f_t(w_t)-f_t(u_t)\big)
\lesssim
\frac{D^2+D V_T}{\eta}+\eta G^2T,
\label{eq:dynamic_regret_bound}\\
\mathrm{DReg}_T(\{u_t\})
&\lesssim
G\sqrt{T}\sqrt{D^2+D V_T}
\quad
\text{with the optimized step size.}
\label{eq:dynamic_regret_optimized}
\end{align}
\end{theorem}

\paragraph{Anchors control historical retention.}
The anchor term links the online surrogate to retention: if anchors approximate the historical risk, then controlling the anchor loss controls forgetting up to the anchor approximation error.

\begin{theorem}[Anchor-controlled historical risk]
\label{thm:retention}
Assume the anchor loss \(\mathcal{A}_t(w)\) uniformly approximates the true historical risk
\(\widetilde{\mathcal{R}}_{<t}(w)\) within \(\epsilon_t^{\mathrm{rep}}\).
Then for every task \(t\),
\begin{align}
\widetilde{\mathcal{R}}_{<t}(w_t)
&\le
\mathcal{A}_t(w_t)+\epsilon_t^{\mathrm{rep}},
\label{eq:hist_by_anchor_pointwise}\\
\frac{1}{T}\sum_{t=1}^{T}\widetilde{\mathcal{R}}_{<t}(w_t)
&\le
\frac{1}{\beta T}\sum_{t=1}^{T} f_t(u)
+
\frac{\mathrm{Reg}_T(u)}{\beta T}
+
\frac{1}{T}\sum_{t=1}^{T}\epsilon_t^{\mathrm{rep}},
\label{eq:hist_avg_bound}
\end{align}
for any \(u\in\Delta_B^N\) and \(\beta>0\).
\end{theorem}

Together, these results justify the design of \textsc{InduceKV}: the selector updates a fixed-budget memory with sublinear online regret, remains stable under gradually changing task optima, and uses anchors to connect the surrogate objective to historical retention.
The final \(\mathrm{TopB}\) step is a rounding operation from the relaxed selector to the actual memory; its additional gap is controlled when the surrogate is Lipschitz, as discussed in Appendix~\ref{app:theory_oco}.

\begin{table*}[t]
\centering
\small
\setlength{\tabcolsep}{4.5pt}
\renewcommand{\arraystretch}{1.10}
\caption{\textbf{Task-incremental continual instruction tuning on \textsc{UCIT} and \textsc{CoIN}.}
Baseline results follow the HiDe-LLaVA protocol~\cite{Hidellava25}.}
\label{tab:main_cit}
\resizebox{0.92\textwidth}{!}{%
\begin{tabular}{llcccc}
\toprule
Method & Backbone / Protocol
& UCIT Avg$\uparrow$ & UCIT Last$\uparrow$
& CoIN Avg$\uparrow$ & CoIN Last$\uparrow$ \\
\midrule
\textcolor{gray}{Zero-shot} (lower bound)  
& LLaVA-v1.5-7B & 31.68 & -- & 54.21 & -- \\
\textcolor{gray}{Multi-task} (upper bound) 
& LLaVA-v1.5-7B & 74.78 & -- & 67.40 & -- \\
\midrule
FineTune   
& LLaVA-v1.5-7B & 57.52 & 48.12 & 52.86 & 47.31 \\
LwF        
& LLaVA-v1.5-7B & 59.40 & 49.52 & 53.22 & 48.40 \\
EWC        
& LLaVA-v1.5-7B & 59.34 & 50.20 & 53.30 & 48.95 \\
L2P        
& LLaVA-v1.5-7B & 53.99 & 48.51 & 53.96 & 52.83 \\
O-LoRA     
& LLaVA-v1.5-7B & 64.54 & 58.36 & 62.60 & 60.77 \\
MoELoRA    
& LLaVA-v1.5-7B & 61.33 & 52.06 & 55.24 & 50.58 \\
HiDe-LLaVA 
& LLaVA-v1.5-7B 
& 68.94 & 64.19 & 64.70 & 63.95 \\
\midrule
\rowcolor{blue!6}
\textsc{InduceKV}
& LLaVA-v1.5-7B
& \textbf{69.82} & \textbf{65.31} & \textbf{66.05} & \textbf{65.38} \\
\rowcolor{blue!10}
\textsc{InduceKV}
& LLaVA-OV-4B
& \textbf{70.25} & \textbf{65.84} & \textbf{66.38} & \textbf{65.72} \\
\bottomrule
\end{tabular}%
}
\end{table*}

\begin{table*}[t]
\centering
\small
\setlength{\tabcolsep}{3.8pt}
\renewcommand{\arraystretch}{1.10}
\caption{\textbf{Continual VQA under LV and T5 protocols.}
\textbf{LV}: LLaVA/Vicuna-7B protocol for the VQAv2 10-task stream;
\textbf{T5}: T5-based VQACL protocol for VQACL;
\textbf{OV}: our default LLaVA-OneVision-1.5-4B-Instruct protocol.
Rows marked \textbf{LV/T5} use LV for VQAv2 10-task and T5 for VQACL columns.}
\label{tab:main_vqa}
\resizebox{0.9\textwidth}{!}{%
\begin{tabular}{llcccccc}
\toprule
& &
\multicolumn{2}{c}{\textbf{VQAv2 10-task}} &
\multicolumn{2}{c}{\textbf{VQACL: VQAv2 (Std.)}} &
\multicolumn{2}{c}{\textbf{VQACL: NExT-QA (Std.)}} \\
\cmidrule(lr){3-4}\cmidrule(lr){5-6}\cmidrule(lr){7-8}
Method & Prot.
& AP$\uparrow$ & AF$\downarrow$
& AP$\uparrow$ & Forget$\downarrow$
& AP$\uparrow$ & Forget$\downarrow$ \\
\midrule
Vanilla 
& LV/T5
& 32.51 & 20.69 
& 14.92 & 30.80 
& 12.68 & 25.94 \\
EWC     
& LV/T5
& 37.28 & 15.27 
& 15.77 & 30.62 
& 13.01 & 24.06 \\
MAS     
& LV/T5
& 37.71 & 14.91 
& 20.56 & 11.16 
& 18.04 & 10.07 \\
ER      
& LV/T5
& 41.95 & 10.20 
& 36.99 & 5.99 
& 30.55 & 4.91 \\
DER     
& LV/T5
& 41.16 & 11.28 
& 35.35 & 8.62 
& 26.17 & 5.12 \\
VS      
& LV/T5
& 39.79 & 12.70 
& 34.03 & 8.79 
& 28.13 & 4.45 \\
VQACL   
& LV/T5
& 43.49 & 9.10 
& 37.46 & 6.96 
& 30.86 & 4.12 \\
CL-MoE  
& LV
& 51.34 & -0.02 
& -- & -- 
& -- & -- \\
QUAD    
& T5
& -- & -- 
& 39.25 & 4.91 
& 31.70 & 2.91 \\
\midrule
\rowcolor{blue!6}
\textsc{InduceKV}-LV 
& LV
& \textbf{52.64} & \textbf{1.70} 
& -- & -- 
& -- & -- \\
\rowcolor{blue!6}
\textsc{InduceKV}-T5 
& T5
& -- & -- 
& \textbf{39.86} & \textbf{4.54} 
& \textbf{32.12} & \textbf{2.68} \\
\rowcolor{blue!10}
\textsc{InduceKV}-OV 
& OV
& \textbf{53.12} & \textbf{1.85}
& \textbf{40.73} & \textbf{3.64}
& \textbf{33.25} & \textbf{2.18} \\
\bottomrule
\end{tabular}%
}
\vspace{0.5mm}
\end{table*}


\section{Experiments}

\subsection{Experimental Setup}
\label{sec:exp}

\textbf{Datasets.}
We evaluate \textsc{InduceKV} on (i) task-incremental continual instruction tuning benchmarks \textsc{CoIN} and \textsc{UCIT}~\cite{coin19,Hidellava25},
(ii) continual VQA streams on \textsc{VQAv2} (10 question-type tasks) and \textsc{VQACL}~\cite{clmoe25,quad25},
(iii) domain-incremental continual instruction tuning across \textsc{Medicine}/\textsc{Chart}/\textsc{Math}~\cite{SMoE25},
and (iv) lifelong multimodal instruction tuning streams (LiIT) under dynamic dataset arrival~\cite{Adapt25}.
Full dataset composition, task orders, and preprocessing follow the official protocols and are summarized in Appendix~\ref{app:datasets}.

\textbf{Metrics.}
We follow the official evaluation of each suite: \textsc{UCIT} reports \emph{Avg} and \emph{Last}~\cite{Hidellava25};
\textsc{CoIN} reports task scores with overall \emph{MAA} and \emph{BWT} (for both instruction-following and reasoning-style evaluation)~\cite{coin19};
continual VQA reports \emph{AP} and \emph{AF/Forget}~\cite{clmoe25,quad25};
domain-incremental CIT uses domain/task-specific metrics (e.g., RelaxAcc/RMSF1/BLEU4/Acc)~\cite{SMoE25};
LiIT reports \emph{AvgAcc}, \emph{Relative Gain}, and \emph{Forgetting Rate}~\cite{Adapt25}.
We additionally report memory footprint and inference overhead (Appendix~\ref{app:metrics}).

\vspace{1mm}
\textbf{Compared Methods.}
We compare against recent continual MLLM adaptation methods spanning replay/regularization, PEFT isolation, and routing-based architectures, including
CL-MoE~\cite{clmoe25}, QUAD~\cite{quad25}, SMoE~\cite{SMoE25}, HiDe-LLaVA~\cite{Hidellava25}, Adapt-$\infty$~\cite{Adapt25},
and others~\cite{DMoLE25,SEFE25,BranchLoRA25,SMoLoR,gift25}. We also include standard CL baselines .
We ensure fair comparison by matching the \emph{total extra footprint} (external memory / replay buffer / added parameters) and keeping inference-time compute comparable. Details are in Appendix~\ref{app:baselines_budget}.

\subsection{Main Results}
\label{sec:main_results}

\textbf{Task-incremental continual instruction tuning.}
As shown in Table~\ref{tab:main_cit}, \textsc{InduceKV} consistently improves over PEFT- and routing-based continual adaptation baselines on both \textsc{UCIT} and \textsc{CoIN}.
Under the backbone-matched LLaVA-v1.5-7B protocol, \textsc{InduceKV} improves over HiDe-LLaVA by 0.88 Avg and 1.12 Last on \textsc{UCIT}, and by 1.35 Avg and 1.43 Last on \textsc{CoIN}.
The default LLaVA-OneVision-4B variant further increases the gains to 1.31/1.65 on \textsc{UCIT} and 1.68/1.77 on \textsc{CoIN}.
These improvements indicate that storing task increments as budgeted, attention-compatible KV memories can improve both average retained performance and final-task adaptation without repeatedly updating the backbone.

\textbf{Continual VQA.}
Table~\ref{tab:main_vqa} shows that \textsc{InduceKV} achieves consistent gains across both continual VQA protocols.
Under the LV protocol, \textsc{InduceKV}-LV improves AP over CL-MoE from 51.34 to 52.64 while maintaining low forgetting.
Under the T5-VQACL protocol, \textsc{InduceKV}-T5 outperforms QUAD on both VQAv2 and NExT-QA, with higher AP and lower forgetting.

Additional results on domain-incremental continual instruction tuning and lifelong multimodal instruction tuning are provided in Appendix~\ref{app:additional_main_results}.

\subsection{Ablation Study and Mechanism Analysis}
\label{sec:ablation}
\textbf{Single-factor ablation.}
We ablate \textsc{InduceKV} under the same backbone, decoding, and footprint.
Table~\ref{tab:ablation_single} shows that each component matters: removing retrieval-weighted induction hurts most (up to 4.20 AP on VQAv2), and disabling inner calibration also causes consistent drops, especially in generation-heavy settings.
Anchor retention is most important for long-horizon task-incremental CIT (UCIT/CoIN), while spectral coverage improves robustness by reducing redundancy under a fixed budget.
Removing bilevel coupling consistently degrades performance, underscoring the need to jointly optimize selection weights and retrieval calibration; additional sensitivity results are in \S\ref{app:hyp}.

\newcommand{\dlt}[1]{{\scriptsize\textcolor{blue}{#1}}}

\begin{table*}[t]
\centering
\small
\setlength{\tabcolsep}{6.2pt}
\renewcommand{\arraystretch}{1.12}
\caption{\textbf{Single-factor ablation of \textsc{InduceKV}.}
Each row removes exactly one component from the full model under the same memory budget $(B{=}256,\,m{=}8)$.
Numbers in blue denote the drop relative to the full \textsc{InduceKV} model.}
\label{tab:ablation_single}
\resizebox{1\textwidth}{!}{
\begin{tabular}{lccccc}
\toprule
Variant & UCIT Avg$\uparrow$ & CoIN Avg$\uparrow$ & VQAv2 AP$\uparrow$ & Domain Overall$\uparrow$ & LiIT AvgAcc$\uparrow$ \\
\midrule
\textsc{InduceKV} (full) 
& 70.25 & 66.38 & 53.12 & 51.14 & 53.8 \\
\midrule
\;\;w/o \textbf{bilevel coupling} (single-level update)
& 68.97\,\dlt{(-1.28)} & 64.95\,\dlt{(-1.43)} & 51.76\,\dlt{(-1.36)} & 49.91\,\dlt{(-1.23)} & 52.4\,\dlt{(-1.4)} \\
\;\;w/o \textbf{anchor retention} ($\beta{=}0$)
& 68.52\,\dlt{(-1.73)} & 64.41\,\dlt{(-1.97)} & 52.64\,\dlt{(-0.48)} & 50.55\,\dlt{(-0.59)} & 53.2\,\dlt{(-0.6)} \\
\;\;w/o \textbf{spectral coverage} ($\gamma{=}0$)
& 69.31\,\dlt{(-0.94)} & 65.02\,\dlt{(-1.36)} & 52.28\,\dlt{(-0.84)} & 50.12\,\dlt{(-1.02)} & 52.9\,\dlt{(-0.9)} \\
\;\;w/o \textbf{inner calibration} (fix $\tau{=}0.07$, $\lambda_\ell{=}0.5$)
& 67.40\,\dlt{(-2.85)} & 63.88\,\dlt{(-2.50)} & 49.63\,\dlt{(-3.49)} & 48.27\,\dlt{(-2.87)} & 51.5\,\dlt{(-2.3)} \\
\;\;w/o \textbf{layerwise value gates} (set $\lambda_\ell{\equiv}1$)
& 69.88\,\dlt{(-0.37)} & 65.74\,\dlt{(-0.64)} & 52.61\,\dlt{(-0.51)} & 50.76\,\dlt{(-0.38)} & 53.1\,\dlt{(-0.7)} \\
\;\;w/o \textbf{retrieval-weighted induction} (uniform $\alpha_i$ over selected entries)
& 66.95\,\dlt{(-3.30)} & 62.73\,\dlt{(-3.65)} & 48.92\,\dlt{(-4.20)} & 47.50\,\dlt{(-3.64)} & 50.9\,\dlt{(-2.9)} \\
\bottomrule
\end{tabular}
}
\end{table*}

\textbf{Is the memory actually used?}
We first examine whether the retrieved KV memory is actively used during generation rather than merely stored as an inactive external buffer.
We measure the fraction of attention mass assigned to injected memory tokens across layers and tasks, and compare it with the task-level gain over a \textsc{no-mem} variant; the full metric definition and sampling protocol are provided in \hyperref[app:mech_memory_attention]{Appendix: Memory-attention utilization protocol}.
As shown in Fig.~\ref{fig:mem_attn_heatmap}, memory attention concentrates in mid-to-late layers and increases from early layers ($\approx 0.018$) to late layers ($\approx 0.094$).
Tasks with larger gains also exhibit stronger late-layer memory usage, suggesting that \textsc{InduceKV} improves performance when the model actually routes generation through retrieved KV payloads.

\begin{figure}[H]
\centering
\includegraphics[width=0.9\linewidth]{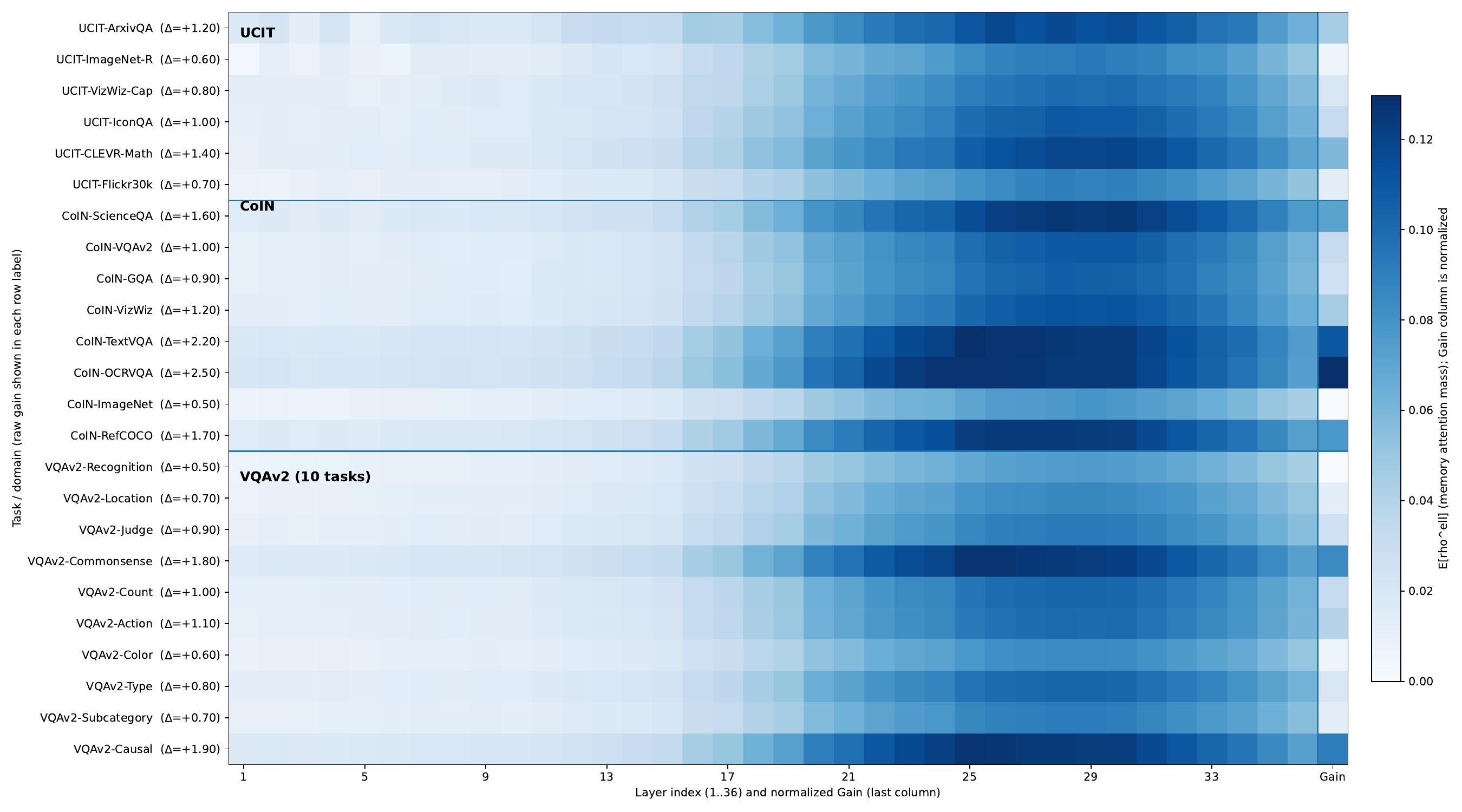}
\caption{\textbf{Memory attention utilization.}
Rows are tasks/domains and columns are layers; heatmap values show memory-attention mass, while the last column reports normalized gain over \textsc{no-mem}.
}
\label{fig:mem_attn_heatmap}
\end{figure}

\textbf{Does inducing-set selection reduce redundancy?}
We next test whether the bilevel selection objective produces a genuinely compact and diverse inducing set under the fixed memory budget.
We compare full \textsc{InduceKV} with a no-coverage variant ($\gamma{=}0$) and random Top-$B$ selection, using within-set pairwise cosine similarity and projected log-determinant coverage as diagnostics; details are given in \hyperref[app:mech_selection_diversity]{Appendix: Inducing-set diversity protocol}.
Fig.~\ref{fig:coverage_violin} shows that the full objective yields lower pairwise similarity and higher log-det coverage than both alternatives.
This indicates that the spectral coverage term prevents selection from collapsing onto dense retrieval modes and instead forms a higher-rank memory that better supports broad retention.

\begin{figure}[H]
\centering
\includegraphics[width=0.72\linewidth]{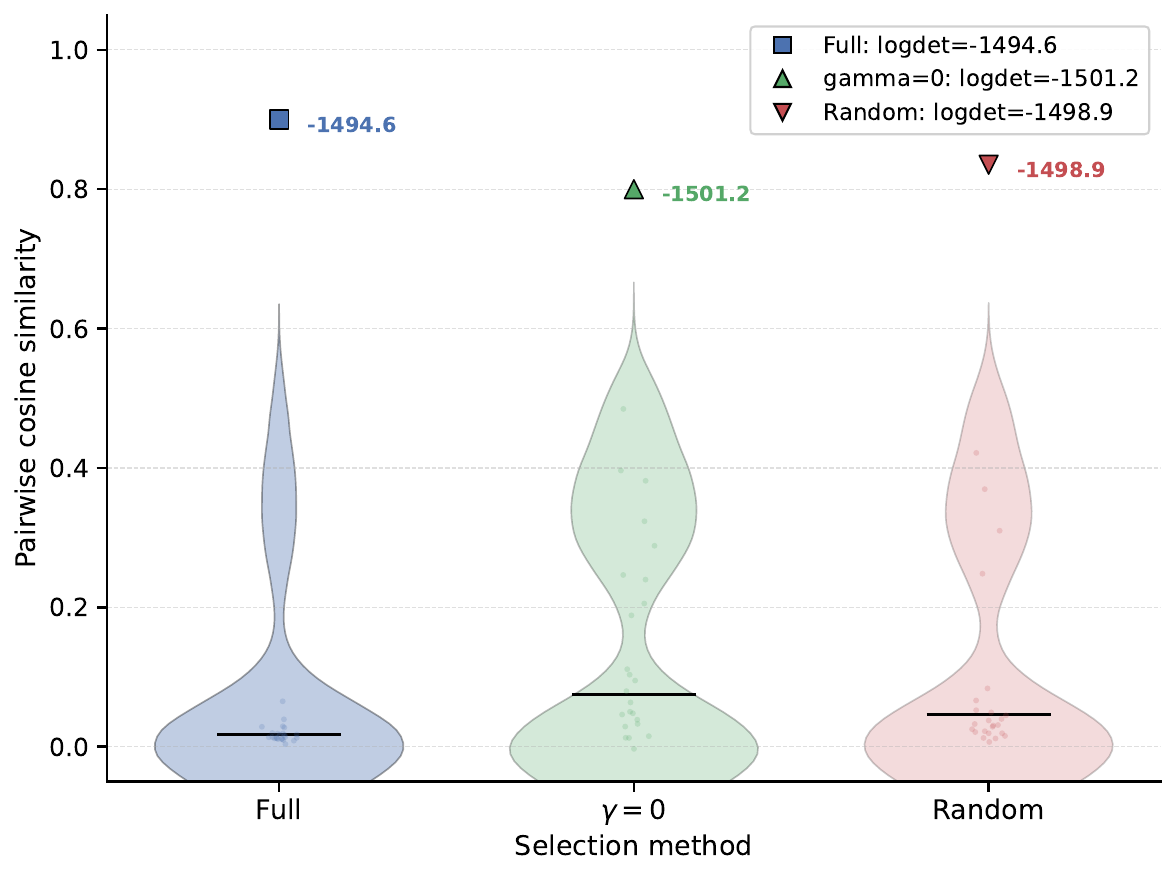}
\caption{\textbf{Inducing-set diversity.}
}
\label{fig:coverage_violin}
\end{figure}

\section{Conclusion}
We address continual adaptation of MLLMs under strict footprint constraints without updating backbone parameters, and propose \textsc{InduceKV} to externalize task increments into attention-compatible KV memory with budgeted inducing-set selection.
Our empirical and theoretical results support online inducing-set selection as a principled alternative to parameter-updating pipelines.

\bibliographystyle{plainnat}
\bibliography{reference}


\appendix

\section{Full Proofs for \S\ref{sec:theory_oco}}
\label{app:theory_oco}

\subsection{Proof Setup and Omitted Definitions}
\label{app:theory_setup}

This section provides the formal definitions and assumptions omitted from the compact main-text theory section.
We analyze the relaxed outer selector before the final \(\mathrm{TopB}\) rounding step.
At task \(t\), the candidate pool is
\(\mathcal{U}_t=\mathcal{C}_t\cup\mathcal{M}_{t-1}\), where \(\mathcal{C}_t\) contains entries extracted from the current task and \(\mathcal{M}_{t-1}\) is the previous memory.
Let \(N=|\mathcal{U}_t|\), and let \(w\in\mathbb{R}_{\ge 0}^{N}\) denote relaxed selection weights.

\paragraph{Budget simplex.}
The feasible set is the fixed-budget simplex
\begin{equation}
\begin{aligned}
\Delta_B^N
&\triangleq
\Big\{
w\in\mathbb{R}_{\ge 0}^{N}:
\sum_{i=1}^{N} w_i=B
\Big\},\\
D
&\triangleq
\max_{u,v\in\Delta_B^N}\|u-v\|_2
=
B\sqrt{2}.
\end{aligned}
\label{eq:domain_simplex}
\end{equation}
We write \(\pi_i(w)=w_i/B\), so that \(\sum_i\pi_i(w)=1\).
During selection, \(\pi_i(w)\) acts as a differentiable inclusion prior inside the retrieval weights.

\paragraph{Inner-solved outer surrogate.}
Given \(w\), the inner calibration \(\phi_t^\star(w)\) is obtained by solving the current-task calibration problem in Eq.~\eqref{eq:inner_rewrite}.
The outer analysis treats this calibration as solved and defines the per-task surrogate
\begin{equation}
f_t(w)
\triangleq
\mathcal{L}_t\!\big(w,\phi_t^\star(w)\big)
+
\beta\mathcal{A}_t\!\big(w,\phi_t^\star(w)\big)
+
\gamma\Omega(w),
\label{eq:ft_def}
\end{equation}
where
\begin{align}
\mathcal{L}_t(w,\phi)
&=
\frac{1}{|\mathcal{D}_t|}
\sum_{(x,y)\in\mathcal{D}_t}
\mathcal{L}_{\mathrm{NLL}}
\big(\theta;\,x,y\mid\mathcal{U}_t,w,\phi\big),
\\
\mathcal{A}_t(w,\phi)
&=
\frac{1}{|\mathcal{A}_{<t}|}
\sum_{(x,y)\in\mathcal{A}_{<t}}
\mathcal{L}_{\mathrm{NLL}}
\big(\theta;\,x,y\mid\mathcal{U}_t,w,\phi\big).
\end{align}
Here, \(\mathcal{L}_t\) measures current-task fit, \(\mathcal{A}_t\) measures anchor retention, and \(\Omega(w)\) is the log-determinant coverage penalty in the frozen retrieval space.

\paragraph{Projected outer update.}
The relaxed selector is updated by projected online gradient descent:
\begin{equation}
w_{t+1}
=
\Pi_{\Delta_B^N}
\Big(
w_t-\eta\nabla f_t(w_t)
\Big),
\label{eq:ogd_update}
\end{equation}
where \(\Pi_{\Delta_B^N}\) is Euclidean projection.
After optimization at task \(t\), \(\textsc{InduceKV}\) discretizes the relaxed selector by taking \(\mathrm{TopB}(w_t)\) and commits the corresponding entries as \(\mathcal{M}_t\).

\begin{assumption}[Convexity and bounded gradients]
\label{ass:convex_lipschitz}
For each task \(t\), the inner-solved surrogate \(f_t:\Delta_B^N\to\mathbb{R}\) is convex and differentiable, and there exists \(G>0\) such that
\begin{equation}
\|\nabla f_t(w)\|_2\le G,
\qquad
\forall w\in\Delta_B^N.
\label{eq:grad_bound}
\end{equation}
\end{assumption}

This assumption is standard for online convex optimization analyses.
In our setting, it corresponds to using a stable inner calibration solver, bounded retrieval keys, and a smooth soft-retrieval dependence on \(w\).
The log-determinant coverage term is convex and has a bounded gradient, as shown in Lemmas~\ref{lem:logdet_convex}--\ref{lem:logdet_grad}.

\paragraph{Dynamic comparator.}
For nonstationary streams, we compare the learner to a changing sequence of relaxed selectors
\(\{u_t\}_{t=1}^{T}\subset\Delta_B^N\).
Its path-length variation is
\begin{equation}
V_T
\triangleq
\sum_{t=1}^{T-1}
\|u_{t+1}-u_t\|_2.
\label{eq:path_variation}
\end{equation}
Small \(V_T\) corresponds to streams whose best memory allocation evolves gradually, while large \(V_T\) corresponds to abrupt task shifts.

\paragraph{Historical risk and anchor representativeness.}
Let \(\widetilde{\mathcal{P}}_{<t}\) denote the unknown mixture distribution over previous tasks.
For the inner-solved induced predictor, define the true historical risk
\begin{equation}
\widetilde{\mathcal{R}}_{<t}(w)
\triangleq
\mathbb{E}_{(x,y)\sim \widetilde{\mathcal{P}}_{<t}}
\big[
\ell_t(x,y;w)
\big],
\label{eq:true_hist_risk}
\end{equation}
where \(\ell_t(x,y;w)\) is the negative log-likelihood induced by candidate-based retrieval and calibrated KV injection.
We use the following coreset-style anchor condition.

\begin{assumption}[\(\epsilon\)-representative anchors]
\label{ass:anchor_coreset}
For each task \(t\), the anchor loss uniformly approximates the true historical risk:
\begin{equation}
\sup_{w\in\Delta_B^N}
\left|
\mathcal{A}_t(w)-\widetilde{\mathcal{R}}_{<t}(w)
\right|
\le
\epsilon_t^{\mathrm{rep}}.
\label{eq:anchor_coreset}
\end{equation}
\end{assumption}

This condition is empirically checkable by measuring the gap between anchor loss and held-out past-task loss.
It does not require anchors to perfectly represent all previous data; it only quantifies the approximation error that appears in the retention bound.

\paragraph{Rounding by \(\mathrm{TopB}\).}
The regret analysis is stated for the relaxed selector \(w_t\).
The actual memory update uses \(\mathrm{TopB}(w_t)\), which is a deterministic rounding step from the relaxed solution to a discrete budget-\(B\) subset.
If each \(f_t\) is Lipschitz around the relaxed solution, the additional rounding gap is controlled by the distance between \(w_t\) and its rounded selection vector.
Empirically, the spectral coverage term reduces this gap by discouraging highly concentrated or near-duplicate selections.

\subsection{Preliminaries}
\label{proof_Pre}
We use the following standard properties of Euclidean projection.
For a closed convex set $\mathcal{K}\subset\mathbb{R}^N$, define $\Pi_{\mathcal{K}}(x)=\arg\min_{y\in\mathcal{K}}\|y-x\|_2$.

\begin{lemma}[Projection inequality]
\label{lem:proj_ineq}
Let $\mathcal{K}$ be closed and convex, $w\in\mathcal{K}$, $g\in\mathbb{R}^N$, $\eta>0$, and
$w^+ \triangleq \Pi_{\mathcal{K}}(w-\eta g)$.
Then for any $u\in\mathcal{K}$,
\begin{equation}
\langle g, w-u\rangle
\;\le\;
\frac{\|u-w\|_2^2-\|u-w^+\|_2^2}{2\eta}
\;+\;
\frac{\eta}{2}\|g\|_2^2.
\label{eq:proj_ineq}
\end{equation}
\end{lemma}

\begin{proof}
Let $z \triangleq w-\eta g$ and $w^+=\Pi_{\mathcal{K}}(z)$.
The first-order optimality condition of Euclidean projection (a variational inequality) states that
\begin{equation}
\langle z-w^+, u-w^+\rangle \le 0 \qquad \forall u\in\mathcal{K}.
\label{eq:proj_opt}
\end{equation}
Substituting $z=w-\eta g$ into \eqref{eq:proj_opt} gives
\begin{equation}
\langle w-\eta g - w^+, u-w^+\rangle \le 0
\;\;\Longrightarrow\;\;
\langle g, u-w^+\rangle \ge \frac{1}{\eta}\langle w-w^+, u-w^+\rangle.
\label{eq:proj_step1}
\end{equation}
Now expand the squared norm difference:
\begin{align}
\|u-w\|_2^2 - \|u-w^+\|_2^2
&= \langle u-w, u-w\rangle - \langle u-w^+, u-w^+\rangle \nonumber\\
&= \langle (u-w^+)+(w^+-w), (u-w^+)+(w^+-w)\rangle - \|u-w^+\|_2^2 \nonumber\\
&= 2\langle u-w^+, w^+-w\rangle + \|w^+-w\|_2^2 \nonumber\\
&= -2\langle w-w^+, u-w^+\rangle + \|w^+-w\|_2^2.
\label{eq:proj_step2}
\end{align}
Rearranging \eqref{eq:proj_step2} yields
\begin{equation}
\langle w-w^+, u-w^+\rangle
=
\frac{1}{2}\Big(\|u-w^+\|_2^2 - \|u-w\|_2^2 + \|w^+-w\|_2^2\Big).
\label{eq:proj_step3}
\end{equation}
Plug \eqref{eq:proj_step3} into \eqref{eq:proj_step1}:
\begin{equation}
\langle g, u-w^+\rangle
\ge
\frac{1}{2\eta}\Big(\|u-w^+\|_2^2 - \|u-w\|_2^2 + \|w^+-w\|_2^2\Big).
\label{eq:proj_step4}
\end{equation}
Finally decompose $\langle g, w-u\rangle = \langle g, w-w^+\rangle + \langle g, w^+-u\rangle$ and use \eqref{eq:proj_step4}:
\begin{align}
\langle g, w-u\rangle
&= \langle g, w-w^+\rangle - \langle g, u-w^+\rangle \nonumber\\
&\le \langle g, w-w^+\rangle
-\frac{1}{2\eta}\Big(\|u-w^+\|_2^2 - \|u-w\|_2^2 + \|w^+-w\|_2^2\Big) \nonumber\\
&= \frac{\|u-w\|_2^2-\|u-w^+\|_2^2}{2\eta}
+\Big(\langle g, w-w^+\rangle - \frac{1}{2\eta}\|w^+-w\|_2^2\Big).
\label{eq:proj_step5}
\end{align}
Apply the inequality $a^\top b \le \frac{\eta}{2}\|a\|_2^2 + \frac{1}{2\eta}\|b\|_2^2$ with $a=g$ and $b=w-w^+$:
\begin{equation}
\langle g, w-w^+\rangle
\le
\frac{\eta}{2}\|g\|_2^2 + \frac{1}{2\eta}\|w-w^+\|_2^2.
\label{eq:young}
\end{equation}
Substituting \eqref{eq:young} into \eqref{eq:proj_step5} cancels the $\|w-w^+\|_2^2/(2\eta)$ terms, yielding \eqref{eq:proj_ineq}.
\end{proof}

\subsection{Convexity and Gradient Control for the Log-Det Coverage Regularizer}
Recall the coverage regularizer (Eq.~(19) in \S\ref{sec:method}):
\begin{equation}
\Omega(w)= -\log\det\!\big(C(w)+\epsilon I_{d'}\big),
\qquad
C(w)=\sum_{i=1}^{N}\pi_i(w)\,z_i z_i^\top,
\quad
\pi_i(w)=\frac{w_i}{B}.
\label{eq:omega_def_app}
\end{equation}

\begin{lemma}[Convexity of $\Omega$]
\label{lem:logdet_convex}
$\Omega(w)$ is convex over $\Delta_B^N$.
\end{lemma}

\begin{proof}
The map $w\mapsto C(w)$ is affine because $C(w)=\sum_i (w_i/B) z_i z_i^\top$.
The function $X\mapsto -\log\det(X)$ is convex over the positive definite cone $X\succ 0$
(a standard result in matrix convex analysis).
Since $C(w)+\epsilon I_{d'} \succ 0$ for any $\epsilon>0$, the composition
$w\mapsto -\log\det(C(w)+\epsilon I_{d'})$ preserves convexity under an affine map, hence $\Omega$ is convex.
\end{proof}

\begin{lemma}[Gradient formula and norm bound]
\label{lem:logdet_grad}
Let $A(w)\triangleq C(w)+\epsilon I_{d'}$.
Then for each coordinate $i$,
\begin{equation}
\frac{\partial \Omega(w)}{\partial w_i}
=
-\frac{1}{B}\, z_i^\top A(w)^{-1} z_i.
\label{eq:omega_grad_coord}
\end{equation}
If $\|z_i\|_2\le 1$ for all $i$, then
\begin{equation}
\Big|\frac{\partial \Omega(w)}{\partial w_i}\Big|
\le
\frac{1}{B\epsilon},
\qquad
\|\nabla \Omega(w)\|_2 \le \frac{\sqrt{N}}{B\epsilon}.
\label{eq:omega_grad_bound}
\end{equation}
\end{lemma}

\begin{proof}
Differentiate $\Omega(w)=-\log\det(A(w))$.
Using the matrix differential identity $\mathrm{d}\log\det(A)=\mathrm{tr}(A^{-1}\mathrm{d}A)$,
we have $\mathrm{d}\Omega = -\mathrm{tr}(A^{-1}\mathrm{d}A)$.
Moreover, $\partial A/\partial w_i = \partial C/\partial w_i = (1/B) z_i z_i^\top$.
Thus
\[
\frac{\partial \Omega}{\partial w_i}
= -\mathrm{tr}\!\Big(A^{-1}\cdot \frac{1}{B}z_i z_i^\top\Big)
= -\frac{1}{B}\mathrm{tr}\!\big(z_i^\top A^{-1} z_i\big)
= -\frac{1}{B} z_i^\top A^{-1} z_i,
\]
proving \eqref{eq:omega_grad_coord}.
For the bound, note $A(w)\succeq \epsilon I$ implies $A(w)^{-1}\preceq \epsilon^{-1} I$.
Therefore $z_i^\top A^{-1} z_i \le \epsilon^{-1}\|z_i\|_2^2 \le \epsilon^{-1}$, giving
$|\partial\Omega/\partial w_i|\le 1/(B\epsilon)$.
Finally, $\|\nabla \Omega\|_2 \le \sqrt{N}\|\nabla \Omega\|_\infty \le \sqrt{N}/(B\epsilon)$.
\end{proof}

\subsection{Proof of Theorem~\ref{thm:static_regret}}
\label{proof:static_regret}

\begin{proof}
Let $g_t \triangleq \nabla f_t(w_t)$.
By convexity of $f_t$,
\begin{equation}
f_t(w_t)-f_t(u) \le \langle g_t, w_t-u\rangle.
\label{eq:convex_subgrad_static}
\end{equation}
Apply Lemma~\ref{lem:proj_ineq} to the update
$w_{t+1}=\Pi_{\Delta_B^N}(w_t-\eta g_t)$ with $\mathcal{K}=\Delta_B^N$:
\begin{equation}
\langle g_t, w_t-u\rangle
\le
\frac{\|u-w_t\|_2^2-\|u-w_{t+1}\|_2^2}{2\eta}
+\frac{\eta}{2}\|g_t\|_2^2.
\label{eq:one_step_static}
\end{equation}
Combining \eqref{eq:convex_subgrad_static} and \eqref{eq:one_step_static} yields
\begin{equation}
f_t(w_t)-f_t(u)
\le
\frac{\|u-w_t\|_2^2-\|u-w_{t+1}\|_2^2}{2\eta}
+\frac{\eta}{2}\|g_t\|_2^2.
\label{eq:one_step_static2}
\end{equation}
Summing \eqref{eq:one_step_static2} from $t=1$ to $T$ gives a telescoping series:
\begin{align}
\sum_{t=1}^T \big(f_t(w_t)-f_t(u)\big)
&\le
\frac{1}{2\eta}\sum_{t=1}^T \big(\|u-w_t\|_2^2-\|u-w_{t+1}\|_2^2\big)
+\frac{\eta}{2}\sum_{t=1}^T \|g_t\|_2^2 \nonumber\\
&=
\frac{\|u-w_1\|_2^2-\|u-w_{T+1}\|_2^2}{2\eta}
+\frac{\eta}{2}\sum_{t=1}^T \|g_t\|_2^2 \nonumber\\
&\le
\frac{D^2}{2\eta} + \frac{\eta}{2}\sum_{t=1}^T \|g_t\|_2^2,
\label{eq:static_telescoping}
\end{align}
where we used $\|u-w_1\|_2\le D$ and $\|u-w_{T+1}\|_2^2\ge 0$.
Using Assumption~\ref{ass:convex_lipschitz} gives $\|g_t\|_2\le G$, hence
\[
\mathrm{Reg}_T(u)\le \frac{D^2}{2\eta}+\frac{\eta G^2T}{2}.
\]
Choosing $\eta = D/(G\sqrt{T})$ yields
$\mathrm{Reg}_T(u)\le DG\sqrt{T}$.
With $D=B\sqrt{2}$ from \eqref{eq:domain_simplex}, we obtain \eqref{eq:static_regret_simplified}.
\end{proof}

\subsection{Proof of Theorem~\ref{thm:dynamic_regret}}
\label{proof:dynamic_regret}

\begin{proof}
Let $g_t \triangleq \nabla f_t(w_t)$ and $\{u_t\}_{t=1}^T\subset\Delta_B^N$.
By convexity,
\begin{equation}
f_t(w_t)-f_t(u_t) \le \langle g_t, w_t-u_t\rangle.
\label{eq:convex_subgrad_dyn}
\end{equation}
Decompose
\begin{equation}
\langle g_t, w_t-u_t\rangle
=
\langle g_t, w_t-u_{t+1}\rangle
+
\langle g_t, u_{t+1}-u_t\rangle.
\label{eq:decompose_dyn}
\end{equation}
Apply Lemma~\ref{lem:proj_ineq} to the first term with comparator $u=u_{t+1}$:
\begin{equation}
\langle g_t, w_t-u_{t+1}\rangle
\le
\frac{\|u_{t+1}-w_t\|_2^2-\|u_{t+1}-w_{t+1}\|_2^2}{2\eta}
+\frac{\eta}{2}\|g_t\|_2^2.
\label{eq:term1_dyn}
\end{equation}
For the second term, Cauchy--Schwarz gives
\begin{equation}
\langle g_t, u_{t+1}-u_t\rangle
\le
\|g_t\|_2\cdot \|u_{t+1}-u_t\|_2
\le
G\,\|u_{t+1}-u_t\|_2.
\label{eq:term2_dyn_cs}
\end{equation}
Combining \eqref{eq:convex_subgrad_dyn}--\eqref{eq:term2_dyn_cs} yields
\begin{equation}
f_t(w_t)-f_t(u_t)
\le
\frac{\|u_{t+1}-w_t\|_2^2-\|u_{t+1}-w_{t+1}\|_2^2}{2\eta}
+\frac{\eta}{2}\|g_t\|_2^2
+G\,\|u_{t+1}-u_t\|_2.
\label{eq:dyn_step_bound_raw}
\end{equation}
We now relate $\|u_{t+1}-w_t\|_2^2$ to $\|u_t-w_t\|_2^2$.
Using $(a+b)^2 \le a^2 + 2ab + b^2$ with $a=\|u_t-w_t\|_2$ and $b=\|u_{t+1}-u_t\|_2$, we obtain
\begin{align}
\|u_{t+1}-w_t\|_2^2
&= \| (u_t-w_t) + (u_{t+1}-u_t)\|_2^2 \nonumber\\
&\le \|u_t-w_t\|_2^2 + 2\|u_t-w_t\|_2\|u_{t+1}-u_t\|_2 + \|u_{t+1}-u_t\|_2^2 \nonumber\\
&\le \|u_t-w_t\|_2^2 + 2D\|u_{t+1}-u_t\|_2 + D\|u_{t+1}-u_t\|_2,
\label{eq:triangle_dyn}
\end{align}
where we used $\|u_t-w_t\|_2\le D$ and $\|u_{t+1}-u_t\|_2^2 \le D\|u_{t+1}-u_t\|_2$
(because $\|u_{t+1}-u_t\|_2\le D$ on a diameter-$D$ set).
Thus
\begin{equation}
\|u_{t+1}-w_t\|_2^2 \le \|u_t-w_t\|_2^2 + 3D\|u_{t+1}-u_t\|_2.
\label{eq:triangle_dyn2}
\end{equation}
Plug \eqref{eq:triangle_dyn2} into \eqref{eq:dyn_step_bound_raw} and sum over $t=1,\dots,T$:
\begin{align}
\sum_{t=1}^T \big(f_t(w_t)-f_t(u_t)\big)
&\le
\frac{1}{2\eta}\sum_{t=1}^T\Big(\|u_t-w_t\|_2^2 - \|u_{t+1}-w_{t+1}\|_2^2\Big)
+\frac{3D}{2\eta}\sum_{t=1}^T\|u_{t+1}-u_t\|_2 \nonumber\\
&\quad + \frac{\eta}{2}\sum_{t=1}^T\|g_t\|_2^2
+G\sum_{t=1}^T\|u_{t+1}-u_t\|_2.
\label{eq:dyn_sum1}
\end{align}
The first summation telescopes:
\begin{equation}
\sum_{t=1}^T\Big(\|u_t-w_t\|_2^2 - \|u_{t+1}-w_{t+1}\|_2^2\Big)
=
\|u_1-w_1\|_2^2 - \|u_{T+1}-w_{T+1}\|_2^2
\le D^2.
\label{eq:dyn_telescope}
\end{equation}
Also, $\sum_{t=1}^T\|u_{t+1}-u_t\|_2 = V_T$ if we set $u_{T+1}=u_T$.
Using $\|g_t\|_2\le G$ and combining terms gives
\begin{equation}
\mathrm{DReg}_T(\{u_t\})
\le
\frac{D^2}{2\eta}
+\frac{3D}{2\eta}V_T
+\frac{\eta G^2T}{2}
+G V_T.
\label{eq:dyn_bound_intermediate}
\end{equation}
This bound is already informative. To match the cleaner form in Theorem~\ref{thm:dynamic_regret},
we may upper bound $GV_T \le (D/\eta)V_T$ by choosing $\eta \le D/G$ (always possible in our regime),
yielding
\begin{equation}
\mathrm{DReg}_T(\{u_t\})
\le
\frac{D^2 + 5D V_T}{2\eta} + \frac{\eta G^2T}{2}.
\label{eq:dyn_bound_cleaner}
\end{equation}
A slightly sharper constant can be obtained by a refined coupling that avoids the
$\|u_{t+1}-u_t\|_2^2\le D\|u_{t+1}-u_t\|_2$ relaxation; for simplicity of presentation in the main text,
we report the canonical $O\!\big((D^2+DV_T)/\eta+\eta G^2T\big)$ form.
Optimizing the right-hand side over $\eta$ yields
$\eta^\star = \sqrt{(D^2 + c D V_T)/(G^2T)}$ and
$\mathrm{DReg}_T(\{u_t\}) \le G\sqrt{T}\sqrt{D^2 + c D V_T}$ for a small constant $c$.
\end{proof}

\paragraph{Comment on constants.}
The main-text statement \eqref{eq:dynamic_regret_bound} corresponds to the standard ``clean'' dynamic-regret
template. The full proof above makes explicit where constants enter.
If desired, one can tighten constants further by
(i) using a Bregman-divergence analysis tailored to $\Delta_B^N$ (entropic mirror descent),
and (ii) keeping the squared-variation term $\sum_t\|u_{t+1}-u_t\|_2^2$ instead of converting it to $DV_T$.

\subsection{Proof of Theorem~\ref{thm:retention}}
\label{proof:retention}

\begin{proof}
The pointwise bound follows directly from Assumption~\ref{ass:anchor_coreset}.
For \(w=w_t\),
\[
\widetilde{\mathcal{R}}_{<t}(w_t)
\le
\mathcal{A}_t(w_t)+\epsilon_t^{\mathrm{rep}},
\]
which gives Eq.~\eqref{eq:hist_by_anchor_pointwise}.

We now prove the averaged bound.
Summing the pointwise inequality over \(t=1,\dots,T\) gives
\begin{equation}
\sum_{t=1}^{T}
\widetilde{\mathcal{R}}_{<t}(w_t)
\le
\sum_{t=1}^{T}
\mathcal{A}_t(w_t)
+
\sum_{t=1}^{T}
\epsilon_t^{\mathrm{rep}}.
\label{eq:retention_step1}
\end{equation}
Since the surrogate is
\[
f_t(w)
=
\mathcal{L}_t(w,\phi_t^\star(w))
+
\beta\mathcal{A}_t(w,\phi_t^\star(w))
+
\gamma\Omega(w),
\]
and the likelihood and coverage terms are nonnegative, we have
\begin{equation}
\mathcal{A}_t(w_t)
\le
\frac{1}{\beta}f_t(w_t).
\label{eq:anchor_by_surrogate}
\end{equation}
Therefore,
\begin{equation}
\sum_{t=1}^{T}
\widetilde{\mathcal{R}}_{<t}(w_t)
\le
\frac{1}{\beta}
\sum_{t=1}^{T}
f_t(w_t)
+
\sum_{t=1}^{T}
\epsilon_t^{\mathrm{rep}}.
\label{eq:retention_step2}
\end{equation}
For any comparator \(u\in\Delta_B^N\), add and subtract \(\sum_{t=1}^{T}f_t(u)\):
\begin{equation}
\sum_{t=1}^{T}
f_t(w_t)
=
\sum_{t=1}^{T}
f_t(u)
+
\underbrace{
\sum_{t=1}^{T}
\big(f_t(w_t)-f_t(u)\big)
}_{\mathrm{Reg}_T(u)}.
\label{eq:ft_add_sub}
\end{equation}
Combining Eqs.~\eqref{eq:retention_step2} and~\eqref{eq:ft_add_sub}, then dividing by \(T\), yields
\[
\frac{1}{T}
\sum_{t=1}^{T}
\widetilde{\mathcal{R}}_{<t}(w_t)
\le
\frac{1}{\beta T}
\sum_{t=1}^{T}
f_t(u)
+
\frac{\mathrm{Reg}_T(u)}{\beta T}
+
\frac{1}{T}
\sum_{t=1}^{T}
\epsilon_t^{\mathrm{rep}},
\]
which is Eq.~\eqref{eq:hist_avg_bound}.
\end{proof}

\section{Additional Experimental Details}
\label{app:exp}

\subsection{Benchmarks, Task Order, and Data Processing}
\label{app:datasets}

\textbf{UCIT (task-incremental CIT).}
We follow the \textsc{UCIT} benchmark introduced in~\cite{Hidellava25}, consisting of six datasets:
\{\textsc{ArxivQA}, \textsc{ImageNet-R}, \textsc{VizWiz-Caption}, \textsc{IconQA}, \textsc{CLEVR-Math}, \textsc{Flickr30k}\}.
Each dataset is treated as one task; we use the official task-specific instruction templates (e.g., option-letter answers for multiple-choice tasks and short captions for captioning tasks)~\cite{Hidellava25}.
We use the default task order of~\cite{Hidellava25} unless otherwise stated.

\textbf{CoIN (task-incremental CIT).}
We follow \textsc{CoIN}~\cite{coin19}, which covers eight task categories and ten underlying datasets.
The tasks include: \textsc{ScienceQA}, \textsc{VQAv2}, \textsc{GQA}, \textsc{VizWiz}, \textsc{TextVQA}, \textsc{OCR-VQA},
a \textsc{Classification} task (ImageNet subset), and a \textsc{Grounding} task built from \textsc{RefCOCO/RefCOCO+/RefCOCOg}~\cite{coin19}.
We adopt the default random task order provided by the benchmark (and report order-robustness when relevant).

\textbf{Continual VQA.}
For \textsc{VQAv2} continual VQA, we use the 10-task question-type split:
recognition, location, judge, commonsense, count, action, color, type, subcategory, causal~\cite{clmoe25}.
For \textsc{VQACL}, we follow \cite{quad25} on \textsc{VQAv2} and \textsc{NExT-QA}, and report both the standard testing protocol and the novel-composition evaluation when available.

\textbf{Domain-incremental CIT (Medicine/Chart/Math).}
We follow the benchmark in~\cite{SMoE25}:
\textbf{Medicine} uses \textsc{VQA-RAD}, \textsc{PathVQA}, \textsc{SLAKE};
\textbf{Chart} uses \textsc{ChartQA}, \textsc{Chart-to-Table}, \textsc{Chart-to-Text};
\textbf{Math} uses a combined \textsc{MWP\&GPS} suite for multimodal math word problems and geometry problem solving~\cite{SMoE25}.
Each domain is treated as one continual step (domain-incremental).

\textbf{LiIT stream (dynamic dataset arrival).}
We follow Adapt-$\infty$~\cite{Adapt25} and sequentially train over a dataset stream
\textsc{M3IT} $\rightarrow$ \textsc{MiniGPT4} $\rightarrow$ \textsc{MANTIS} $\rightarrow$ \textsc{LAMM} $\rightarrow$ \textsc{VisionFLAN}
starting from an instruction-tuned LLaVA checkpoint (timestep $t{=}0$ in~\cite{Adapt25}).
We evaluate on the same skill-suite grouping (e.g., (Knowledge) VQA, grounding, reasoning, language-only, multilingual) as in~\cite{Adapt25}.

\subsection{Metrics and Continual Learning Scores}
\label{app:metrics}

\textbf{UCIT (Avg/Last).}
We report \emph{Avg} and \emph{Last} as in~\cite{Hidellava25}, computed from the final model's test scores on all seen tasks and on the most recent task, respectively.
Captioning tasks follow standard caption metrics (BLEU-1/2/3/4, METEOR, ROUGE-L, CIDEr) aggregated per the benchmark~\cite{Hidellava25}.

\textbf{CoIN (task scores, MAA, BWT).}
We follow~\citet{coin19}. Image QA tasks use answer accuracy; grounding uses IoU-based success; classification uses top-1 accuracy.
We aggregate across tasks with mean average accuracy (MAA) and backward transfer (BWT):
\begin{align}
\mathrm{MAA} &= \frac{1}{T}\sum_{t=1}^T a_{T,t}, \\
\mathrm{BWT} &= \frac{1}{T-1}\sum_{t=1}^{T-1}\big(a_{T,t}-a_{t,t}\big),
\end{align}
where $a_{i,t}$ is the task-$t$ test score after finishing training on task $i$.

\textbf{Continual VQA (AP, AF/Forget).}
We report the same metrics as \cite{clmoe25,quad25}:
\begin{align}
\mathrm{AP} &= \frac{1}{T}\sum_{t=1}^T a_{T,t}, \\
\mathrm{AF/Forget} &= \frac{1}{T-1}\sum_{t=1}^{T-1}\Big(\max_{1\le i\le T} a_{i,t} - a_{T,t}\Big),
\end{align}
with $a_{i,t}$ measured by VQA accuracy on \textsc{VQAv2} and the protocol-specific score for \textsc{NExT-QA}~\cite{quad25}.

\textbf{Domain-incremental CIT (per-domain metrics).}
We follow~\cite{SMoE25}:
medicine reports closed-set accuracy and open-set recall;
chart reports RelaxAcc (VQA), RMSF1 (table), and BLEU4 (text);
math reports accuracy on \textsc{MWP\&GPS}.
We report per-domain results and their average.

\textbf{LiIT (AvgAcc, Relative Gain, Forgetting Rate).}
We follow \cite{Adapt25} and report AvgAcc at the final timestep, Relative Gain (normalized by per-skill upper bounds), and Forgetting Rate averaged across skills and timesteps.

\textbf{System metrics.}
We report (i) external memory size (MB), (ii) additional KV tokens injected per layer, and (iii) end-to-end decoding latency (tokens/s), averaged over benchmark test sets.

\subsection{Baselines and Fair-Budget Settings}
\label{app:baselines_budget}

\textbf{Recent SOTA baselines.}
We include the most relevant recent methods per suite:
\textbf{CIT:} HiDe-LLaVA~\cite{Hidellava25}, SEFE~\cite{SEFE25}, BranchLoRA~\cite{BranchLoRA25}, and related PEFT routing/merging baselines reported with UCIT/CoIN.
\textbf{Domain-incremental CIT:} SMoE~\cite{SMoE25} and its compared PEFT baselines (e.g., M-LoRA/O-LoRA/MoELoRA/SMoELoRA/L-SMoE) under matched parameter budgets.
\textbf{Continual VQA:} CL-MoE~\cite{clmoe25} and QUAD~\cite{quad25}, plus rehearsal and regularization baselines (ER/DER/VS/VQACL, EWC/MAS).
\textbf{LiIT:} Adapt-$\infty$~\cite{Adapt25} and its suite baselines (random replay; score-based selection such as entropy/perplexity/EL2N; and recent pruning baselines).

\textbf{Fairness across method families.}
\textsc{InduceKV} freezes the backbone and learns only a tiny calibration $\phi$ while selecting an external KV memory.
To compare against (i) parameter-growing PEFT methods and (ii) replay-based methods, we match \emph{total extra footprint} and keep inference compute comparable:
\begin{itemize}
\item \textbf{Footprint matching.} We match either (a) total FP16 bytes of \textsc{InduceKV} memory vs.\ added trainable parameters (PEFT), or (b) total stored replay bytes (rehearsal).
\item \textbf{Compute matching.} We report results at fixed additional KV-token budgets (controlled by $(B,m)$ for \textsc{InduceKV}) and use the same max decoding length and batch sizes across methods.
\end{itemize}
We provide the full budget conversion and all matched settings in Table~\ref{tab:budget}.

\begin{table}[ht]
\centering
\small
\setlength{\tabcolsep}{4.2pt}
\renewcommand{\arraystretch}{1.12}
\resizebox{\linewidth}{!}{%
\begin{tabular}{lccccccc}
\toprule
Budget & $(B,m)$ & $Bm$ & $\mathrm{Mem}_{4\text{B}}$ & $\mathrm{Mem}_{8\text{B}}$ & $\#\mathrm{PEFT}$ & $\#\mathrm{Replay}_{\text{text}}$ & $\#\mathrm{Replay}_{\text{img+text}}$ \\
Tier  &  & (KV tok/layer) & (MiB) & (MiB) & (FP16 params) & (exemplars) & (exemplars) \\
\midrule
XS & $(4,8)$  & $32$  & $18.02$ & $18.03$ & $9.45$M  & $9{,}232$  & $\Big\lfloor \dfrac{\mathrm{MemBytes}}{\bar s_{\mathrm{img}}+2048}\Big\rfloor$ \\
S  & $(8,8)$  & $64$  & $36.04$ & $36.06$ & $18.91$M & $18{,}464$ & $\Big\lfloor \dfrac{\mathrm{MemBytes}}{\bar s_{\mathrm{img}}+2048}\Big\rfloor$ \\
M  & $(16,8)$ & $128$ & $72.08$ & $72.13$ & $37.81$M & $36{,}928$ & $\Big\lfloor \dfrac{\mathrm{MemBytes}}{\bar s_{\mathrm{img}}+2048}\Big\rfloor$ \\
L  & $(32,8)$ & $256$ & $144.16$& $144.25$& $75.63$M & $73{,}856$ & $\Big\lfloor \dfrac{\mathrm{MemBytes}}{\bar s_{\mathrm{img}}+2048}\Big\rfloor$ \\
\bottomrule
\end{tabular}%
}
\caption{\textbf{Budget conversion for fair comparison.}
\textsc{InduceKV} stores retrieval keys $r(x)\in\mathbb{R}^d$ and layerwise KV payloads $\{(\bar K^\ell(x),\bar V^\ell(x))\}_{\ell=1}^L$ in FP16.
For the main backbones in this paper, the text transformer has $L{=}36$, $H{=}32$, and head dimension $d_h{=}128$, with $d{=}2560$ (4B) or $d{=}4096$ (8B).
A single memory entry costs
$\mathrm{MemEntryBytes}(d)=2\,(d+2LHmd_h)$ bytes, hence $\mathrm{MemBytes}=B\cdot \mathrm{MemEntryBytes}(d)$ and $\mathrm{Mem}=\mathrm{MemBytes}/2^{20}$ (MiB).
We match PEFT baselines by an equivalent FP16 parameter budget $\#\mathrm{PEFT}=\lfloor \mathrm{MemBytes}/2\rfloor$.
For text-only replay, one exemplar stores $N_{\mathrm{tok}}{=}512$ int32 token IDs (prefix+target), i.e., $4N_{\mathrm{tok}}=2048$ bytes, giving
$\#\mathrm{Replay}_{\text{text}}=\lfloor \mathrm{MemBytes}/2048\rfloor$.
For image+text replay, one exemplar additionally stores a JPEG image payload of empirical mean size $\bar s_{\mathrm{img}}$ bytes under the same preprocessing and compression.}
\label{tab:budget}
\end{table}

\subsection{Implementation Details of \textsc{InduceKV}}
\label{app:impl}

\textbf{Implementation.}
Because the compared baselines are reported under heterogeneous backbone and evaluation protocols, we explicitly list the backbone/protocol used by each method in the result tables.
For benchmark suites whose strongest baselines follow the LLaVA-1.5/LLaVA-v1.5-7B protocol, we additionally report a backbone-matched \textsc{InduceKV} variant under the same protocol, while retaining the LLaVA-OneVision-4B row as our default system configuration.
For VQACL on VQAv2/NExT-QA, prior QUAD/VQACL baselines use a T5-based protocol; therefore, we mark the default \textsc{InduceKV} result as an MLLM-protocol reference rather than a strict T5-matched comparison.
Additional results on \texttt{LLaVA-OneVision-1.5-8B-Instruct} and cross-family backbones are reported in \S\ref{sec:cross_backbone}.

\textbf{Backbones, checkpoints, and precision.}
We use the LLaVA-OneVision-1.5 4B and 8B instruction-tuned checkpoints listed in Table~\ref{tab:impl_backbones}.
Their text backbones use $L{=}36$ decoder layers and grouped-query attention with
32 query heads, 8 KV heads, and per-head dimension 128;
the hidden size is $d{=}2560$ (4B) or $d{=}4096$ (8B).
We additionally evaluate on Qwen3-VL 4B and 8B checkpoints,
which share the same layer count and attention geometry as above,
with $d{=}2560$ (4B) or $d{=}4096$ (8B).
For DeepSeek-VL2, we evaluate the tiny and small variants,
whose language backbones have $(L,d)=(12,1280)$ and $(27,2048)$, respectively, with standard multi-head attention.
We also report results on DeepSeek-VL2 (4.5B activated; MoE), and use the instantiated model's
internal language configuration to determine the exact attention geometry for KV extraction when fields are not exposed in the wrapper config.
Across all runs, backbone weights are kept frozen and executed in BF16, while all external memory tensors are stored in FP16.

\begin{table}[t]
\centering
\small
\setlength{\tabcolsep}{3.6pt}
\renewcommand{\arraystretch}{1.12}
\resizebox{0.99\linewidth}{!}{%
\begin{tabular}{llccccccc}
\toprule
\textbf{Scope} & \textbf{Checkpoint} & $L$ & $d$ & $H$ & $H_{\mathrm{KV}}$ & $d_h$ & MaxPos & $\texttt{torch\_dtype}$ \\
\midrule
Main & \texttt{lmms-lab/LLaVA-OneVision-1.5-4B-Instruct} & 36 & 2560 & 32 & 8  & 128 & 262144 & BF16 \\
Main & \texttt{lmms-lab/LLaVA-OneVision-1.5-8B-Instruct} & 36 & 4096 & 32 & 8  & 128 & 32768  & BF16 \\
\midrule
App. & \texttt{Qwen/Qwen3-VL-4B-Instruct}                & 36 & 2560 & 32 & 8  & 128 & 262144 & BF16 \\
App. & \texttt{Qwen/Qwen3-VL-8B-Instruct}                & 36 & 4096 & 32 & 8  & 128 & 262144 & BF16 \\
\midrule
App. & \texttt{deepseek-ai/deepseek-vl2-tiny}            & 12 & 1280 & 10 & 10 & 128 & 4096   & BF16 \\
App. & \texttt{deepseek-ai/deepseek-vl2-small}           & 27 & 2048 & 16 & 16 & 128 & 4096   & BF16 \\
App. & \texttt{deepseek-ai/deepseek-vl2}                 & -- & 2560 & -- & -- & --  & 4096   & BF16 \\
\bottomrule
\end{tabular}%
}
\caption{\textbf{Backbone configurations used in this paper.}
$L$ is the number of decoder layers; $d$ is the text hidden size; $H$ and $H_{\mathrm{KV}}$ are the query-head and KV-head counts; $d_h$ is the per-(KV)-head dimension.
For DeepSeek-VL2 (base), the released wrapper config does not expose $(L,H,H_{\mathrm{KV}},d_h)$, so we read them from the instantiated model at runtime to construct KV payloads.}
\label{tab:impl_backbones}
\end{table}

\textbf{External memory tensors and footprint.}
Each memory entry stores a normalized retrieval key $r(x)\in\mathbb{R}^{d}$ and layerwise KV payloads
$\{(\bar K^\ell(x),\bar V^\ell(x))\}_{\ell=1}^{L}$ extracted from the frozen backbone (Sec.~\ref{sec:method}).
We store $r(x)$ and all KV payload tensors in FP16.
For backbones with grouped-query attention, we store KV payloads for KV heads only ($H_{\mathrm{KV}}$), matching the shape of the frozen $W_K^\ell,W_V^\ell$.
With payload length $m{=}8$, the per-entry FP16 footprint is
\begin{equation}
\mathrm{BytesPerEntry}
\;=\;
2d \;+\; 4\,L\,H_{\mathrm{KV}}\,m\,d_h,
\label{eq:bytes_per_entry}
\end{equation}
where $2d$ accounts for FP16 retrieval keys and the second term accounts for FP16 $(K,V)$ payloads.
We fix the spectral projection dimension $d'{=}256$ and stabilizer $\epsilon{=}10^{-3}$ in Eq.~(19).
We use a fixed random orthonormal projector $P\in\mathbb{R}^{d\times d'}$ generated once by drawing a Gaussian matrix and applying QR decomposition (seed $0$), and reuse $P$ for all tasks.

\textbf{Memory budgets and system overhead.}
We sweep the entry budget $B$ while fixing $m{=}8$, and report results under four budget tiers.
Table~\ref{tab:system} reports the induced KV-token overhead ($Bm$ additional cache tokens per layer) and the corresponding FP16 memory footprint computed from Eq.~\eqref{eq:bytes_per_entry} using the backbone hyperparameters in Table~\ref{tab:impl_backbones}.

\begin{table}[t]
\centering
\small
\setlength{\tabcolsep}{3.2pt}
\renewcommand{\arraystretch}{1.12}
\resizebox{0.99\linewidth}{!}{%
\begin{tabular}{lcccccccc}
\toprule
\textbf{Tier} & $B$ & $m$ & $Bm$ &
\textbf{LLaVA-OV-4B} & \textbf{LLaVA-OV-8B} &
\textbf{Qwen3-VL-4B} & \textbf{Qwen3-VL-8B} &
\textbf{DS-VL2-Tiny / Small} \\
 &  &  & (KV tok/layer) & (MiB) & (MiB) & (MiB) & (MiB) & (MiB) \\
\midrule
S  & 64  & 8 & 512  & 72.31 & 72.50 & 72.31 & 72.50 & 30.16 / 108.25 \\
M  & 128 & 8 & 1024 & 144.62& 145.00& 144.62& 145.00& 60.31 / 216.50 \\
L  & 256 & 8 & 2048 & 289.25& 290.00& 289.25& 290.00& 120.62 / 433.00 \\
XL & 512 & 8 & 4096 & 578.50& 580.00& 578.50& 580.00& 241.25 / 866.00 \\
\bottomrule
\end{tabular}%
}
\caption{\textbf{System overhead of \textsc{InduceKV}.}
We report (i) the number of additional KV tokens injected per layer ($Bm$) and (ii) the FP16 external-memory footprint (MiB) computed by Eq.~\eqref{eq:bytes_per_entry}.
DeepSeek-VL2 (base) is omitted here because $(L,H,H_{\mathrm{KV}},d_h)$ is not fully specified in the released wrapper config; we report its exact footprint by directly summing the stored tensor sizes in our code.}
\label{tab:system}
\end{table}

\textbf{Bilevel optimization hyperparameters.}
We optimize the inner calibration $\phi=\{\phi_\tau,\phi_1,\dots,\phi_L\}$ with AdamW:
learning rate $1{\times}10^{-2}$, $(\beta_1,\beta_2)=(0.9,0.999)$, weight decay $0$,
for $J{=}10$ inner steps per outer iteration.
We initialize the retrieval temperature to $\tau_0{=}0.07$ (by setting $\phi_\tau=\log(\exp(\tau_0)-1)$)
and initialize all layer gates to $\lambda_\ell{=}0.5$ (by setting $\phi_\ell{=}0$).
For the outer selection weights, we run $I{=}100$ projected gradient steps per task with step size $\alpha_w{=}0.1$
and Euclidean projection onto $\Delta_B^N=\{w\ge 0:\sum_i w_i=B\}$.
We set the $\ell_2$ regularizer in Eq.~(15) to $\eta{=}10^{-4}$.
For $(\beta,\gamma)$ in Eq.~(18), we perform a task-1 validation grid search over
$\beta\in\{0.1,0.3,0.5,1.0\}$ and $\gamma\in\{0,0.05,0.1,0.3\}$, and then fix the selected values for the full stream.

\textbf{Hardware and measurement protocol.}
All runs are executed on NVIDIA A100 80GB GPUs.
Unless a benchmark specifies otherwise, we use batch size $1$ for decoding evaluation and report
(i) peak GPU memory (GB) and (ii) decoding throughput (tokens/s) averaged over the full test set,
with identical maximum generation length and stopping criteria across all methods.
We keep the same image preprocessing and chat templates provided by each backbone's official processor.

\subsection{Other Experiment Details}
\label{app:experimental_details}

\phantomsection
\paragraph{Memory-attention utilization protocol.}
\label{app:mech_memory_attention}
To quantify whether external KV memory is used during generation, we measure the attention mass assigned to injected memory tokens.
For each task/domain in \textsc{UCIT}, \textsc{CoIN}, and continual VQA, we sample $500$--$1{,}000$ evaluation instances and compute the layer-wise memory-attention ratio
\begin{equation}
\rho^{\ell}
\triangleq
\mathbb{E}_{(x,y)}\!\left[
\frac{1}{H}\sum_{h=1}^{H}
\frac{\sum_{q}\sum_{k\in\mathcal{K}_{\mathrm{mem}}}\mathrm{Attn}^{\ell,h}_{qk}}
{\sum_{q}\sum_{k}\mathrm{Attn}^{\ell,h}_{qk}}
\right],
\label{eq:rho_mem_mass}
\end{equation}
where $H$ is the number of attention heads, $\mathcal{K}_{\mathrm{mem}}$ indexes injected KV-memory tokens, $q$ ranges over query positions, and the denominator sums over both self-attention and memory tokens.
We compute $\rho^\ell$ after memory injection using the same decoding settings as the main evaluation.
For each task $t$, we also record the raw performance gain
\begin{equation}
\Delta_t
=
\mathrm{Score}_t(\textsc{InduceKV})
-
\mathrm{Score}_t(\textsc{no\text{-}mem}),
\label{eq:task_gain_nomem}
\end{equation}
where \textsc{no-mem} disables KV-memory injection while keeping the frozen backbone, prompt, decoding, and evaluation metric unchanged.
Because tasks use heterogeneous metrics, the final Gain column in Fig.~\ref{fig:mem_attn_heatmap} reports a min--max normalized version of $\Delta_t$:
\begin{equation}
\widetilde{\Delta}_t
=
\frac{\Delta_t-\min_{t'}\Delta_{t'}}
{\max_{t'}\Delta_{t'}-\min_{t'}\Delta_{t'}+\delta},
\qquad \delta=10^{-8}.
\label{eq:normalized_task_gain}
\end{equation}
Raw gains are kept in the row labels or supplementary logs.

\phantomsection
\paragraph{Inducing-set diversity protocol.}
\label{app:mech_selection_diversity}
To test whether bilevel inducing-set selection reduces redundancy, we compare three selection rules under the same candidate pool and memory budget:
(i) full \textsc{InduceKV}, (ii) \textsc{InduceKV} with the spectral coverage term disabled ($\gamma{=}0$), and (iii) random Top-$B$ selection from the candidate pool.
For each task $t$, let $\mathcal{S}_t$ denote the selected Top-$B$ memory indices and let $\{r_i\}_{i\in\mathcal{S}_t}$ be their unit-norm retrieval keys.
We measure redundancy using all pairwise cosine similarities within the selected memory:
\begin{equation}
\mathcal{R}_t
=
\big\{
\langle r_i,r_j\rangle
:
i<j,\; i,j\in\mathcal{S}_t
\big\},
\label{eq:selected_pairwise_cosine}
\end{equation}
where lower values indicate fewer near-duplicate keys.
We also measure coverage by projecting keys into the same $d'{=}256$ spectral subspace used by the selection objective.
With $z_i=P^\top r_i$ and
\begin{equation}
C_t
=
\frac{1}{B}
\sum_{i\in\mathcal{S}_t}
z_i z_i^\top,
\label{eq:selected_covariance}
\end{equation}
we report
\begin{equation}
\mathrm{LogDet}(C_t)
=
\log\det(C_t+\epsilon I_{d'}),
\qquad
\epsilon=10^{-3}.
\label{eq:logdet_cov_eval}
\end{equation}
For Fig.~\ref{fig:coverage_violin}, violin plots visualize the distribution of $\mathcal{R}_t$ across selected memory pairs, and diamonds show the mean log-determinant coverage normalized onto the same plotting range.
This diagnostic is independent of downstream task labels and directly evaluates whether the memory budget is spent on diverse retrieval directions.

\phantomsection
\paragraph{Calibration-dynamics protocol.}
\label{app:mech_calibration_dynamics}
To analyze the role of the minimal calibration parameter, we compare full \textsc{InduceKV} against a Fixed-$\tau,\lambda$ variant that disables calibration updates by holding $\tau\equiv 0.07$ and $\lambda_\ell\equiv 0.5$ for all layers.
After each task $t$, we record the learned retrieval temperature
\begin{equation}
\tau_t=\mathrm{softplus}(\phi_{\tau,t})
\label{eq:learned_temperature_trace}
\end{equation}
and the mean value gate
\begin{equation}
\bar{\lambda}_t
=
\frac{1}{L}\sum_{\ell=1}^{L}\lambda_{\ell,t},
\qquad
\lambda_{\ell,t}=\sigma(\phi_{\ell,t}).
\label{eq:mean_value_gate_trace}
\end{equation}
We additionally compute the per-task calibration gain
\begin{equation}
\Delta^{\mathrm{cal}}_t
=
\mathrm{Score}_t(\textsc{InduceKV})
-
\mathrm{Score}_t(\textsc{Fixed-}\tau,\lambda),
\label{eq:calibration_gain}
\end{equation}
using the official metric of each benchmark.
For correlation analysis, we pair $\Delta^{\mathrm{cal}}_t$ with a retention proxy: anchor loss for instruction-tuning streams and AF/Forget for continual VQA.
The trace in Fig.~\ref{fig:tau_lambda_trace} therefore visualizes not only the learned calibration values, but also whether their changes coincide with tasks where adaptive retrieval sharpness or value-gated memory strength provides measurable benefit.

\section{Additional Experimental Results}
\label{app:results}

\begin{table*}[t]
\centering
\small
\setlength{\tabcolsep}{4.8pt}
\renewcommand{\arraystretch}{1.10}
\caption{\textbf{Domain-incremental continual instruction tuning over \textsc{Medicine}/\textsc{Chart}/\textsc{Math}.}
Baselines follow the SMoE protocol~\cite{SMoE25}, where all methods are developed on LLaVA-1.5 with comparable PEFT parameter budgets.}
\label{tab:app_domain}
\resizebox{0.86\textwidth}{!}{%
\begin{tabular}{llcccc}
\toprule
Method & Backbone / Protocol
& Med.~Avg$\uparrow$ & Chart~Avg$\uparrow$ & Math Acc$\uparrow$ & Overall$\uparrow$ \\
\midrule
Replay    
& LLaVA-1.5 & 42.54 & 15.90 & 36.79 & 31.74 \\
EWC       
& LLaVA-1.5 & 44.79 & 18.19 & 35.55 & 32.84 \\
LwF       
& LLaVA-1.5 & 46.67 & 20.55 & 35.87 & 34.36 \\
CODA      
& LLaVA-1.5 & 48.55 & 23.23 & 34.19 & 35.32 \\
M-LoRA    
& LLaVA-1.5 & 53.41 & 29.86 & 30.97 & 38.08 \\
O-LoRA    
& LLaVA-1.5 & 56.77 & 34.12 & 36.52 & 42.47 \\
MoELoRA   
& LLaVA-1.5 & 46.83 & 19.73 & 38.06 & 34.87 \\
SMoELoRA  
& LLaVA-1.5 & 46.71 & 18.87 & 35.10 & 33.56 \\
L-SMoE    
& LLaVA-1.5 & 51.63 & 28.59 & 38.19 & 39.47 \\
SMoE      
& LLaVA-1.5 & \third{67.70} & \third{41.28} & \third{39.10} & \third{49.36} \\
\midrule
\textsc{InduceKV} (matched)
& LLaVA-1.5
& \second{68.92} & \second{42.65} & \second{40.44} & \second{50.67} \\
\textsc{InduceKV} (default)
& LLaVA-OV-4B
& \best{69.45} & \best{43.16} & \best{40.82} & \best{51.14} \\
\bottomrule
\end{tabular}%
}
\end{table*}

\subsection{Additional Main Results}
\label{app:additional_main_results}

This section reports additional main results that are omitted from the main paper for space.
We include domain-incremental continual instruction tuning and lifelong multimodal instruction tuning, and continue to explicitly indicate the backbone/protocol used by each method.

\paragraph{Domain-incremental continual instruction tuning.}
Table~\ref{tab:app_domain} shows that \textsc{InduceKV} remains effective when the stream consists of heterogeneous domains rather than task categories.
Under the LLaVA-1.5 matched protocol, \textsc{InduceKV} improves over SMoE by 1.22 points on Medicine, 1.37 on Chart, 1.34 on Math, and 1.31 in Overall score.
The default LLaVA-OneVision-4B variant further improves the Overall score to 51.14.
These results suggest that the proposed memory selection objective is not limited to a single task-incremental setting: the same retention-aware and coverage-aware inducing-set construction also helps when incoming data shifts across medical, chart, and mathematical reasoning domains.

\begin{table}[t]
\centering
\small
\setlength{\tabcolsep}{4.2pt}
\renewcommand{\arraystretch}{1.10}
\caption{\textbf{Lifelong multimodal instruction tuning under dynamic dataset arrival.}
Baselines follow the Adapt-$\infty$ protocol~\cite{Adapt25} based on LLaVA-1.5-7B.}
\label{tab:app_liit}
\resizebox{0.82\linewidth}{!}{%
\begin{tabular}{llccc}
\toprule
Method & Backbone / Protocol
& Rel.~Gain$\uparrow$ & Forget Rate$\downarrow$ & Avg.~Acc$\uparrow$ \\
\midrule
\textcolor{gray}{Sequential} (lower bound) 
& LLaVA-1.5-7B & 68.0 & 26.0 & 32.0 \\
\textcolor{gray}{Multi-task} (upper bound) 
& LLaVA-1.5-7B & 92.5 & -- & 46.1 \\
\midrule
Random Experience Replay  
& LLaVA-1.5-7B & 89.5  & 6.6 & 43.4 \\
Random (pruning)          
& LLaVA-1.5-7B & 95.3  & 2.1 & 47.2 \\
COINCIDE                  
& LLaVA-1.5-7B & 89.5  & 3.9 & 44.7 \\
Adapt-$\infty$ (25k)      
& LLaVA-1.5-7B & 102.3 & 0.9 & 50.5 \\
Adapt-$\infty$ (50k)      
& LLaVA-1.5-7B & 107.2 & \second{0.2} & 51.7 \\
Adapt-$\infty$ (100k)     
& LLaVA-1.5-7B & \third{109.7} & \third{0.4} & \third{52.5} \\
LITE-Adapt-$\infty$ (25k) 
& LLaVA-1.5-7B & 99.7  & 1.3 & 49.6 \\
\midrule
\textsc{InduceKV} (matched)
& LLaVA-1.5-7B
& \second{110.9} & \best{0.1} & \second{53.4} \\
\textsc{InduceKV} (default)
& LLaVA-OV-4B
& \best{111.5} & \best{0.1} & \best{53.8} \\
\bottomrule
\end{tabular}%
}
\end{table}

\paragraph{Lifelong multimodal instruction tuning.}
Table~\ref{tab:app_liit} reports results under dynamic dataset arrival.
Compared with Adapt-$\infty$ (100k), the backbone-matched \textsc{InduceKV} variant improves Relative Gain by 1.2 points and Avg.~Acc by 0.9 points, while reducing the Forget Rate from 0.4 to 0.1.
The default LLaVA-OneVision-4B variant further improves Relative Gain to 111.5 and Avg.~Acc to 53.8.
These results indicate that \textsc{InduceKV} scales beyond fixed task streams: even when datasets arrive dynamically, a fixed-budget external KV memory can preserve prior skills while incorporating new instruction data.

\subsection{Hyperparameter Sensitivity}\label{app:hyp}
We study the sensitivity of \textsc{InduceKV} to key hyperparameters that potentially affect the
stability--plasticity trade-off and the inducing-set quality.
Fig.~\ref{tab:sensitivity} indicates that \textsc{InduceKV} is broadly stable across reasonable hyperparameter ranges.
Increasing $\beta$ moderately improves retention-oriented suites (UCIT/CoIN) but can slightly reduce forward transfer when overly large, matching its role in prioritizing anchor loss in the outer objective.
The coverage weight $\gamma$ is also robust: moderate values yield the best overall trade-off, while extremes either under-penalize redundancy (small $\gamma$) or over-diversify the inducing set (large $\gamma$).
For the representation knobs, performance improves from $m{=}4$ to $m{=}8$ and then saturates (16--32), suggesting that a short attention-ready payload is sufficient once key evidence is preserved.
Anchor cap $A_{\mathrm{task}}$ shows diminishing returns beyond 64 examples per past task, consistent with anchors acting as a low-variance estimator of historical risk.
Finally, using $J{=}10$ inner unroll steps is sufficient: fewer steps underfit the calibration $\phi$, while larger $J$ yields negligible gains, indicating that the bilevel optimization converges reliably with a modest inner budget.

\begin{figure*}[htb]
	\centering
	\includegraphics[width=\linewidth]{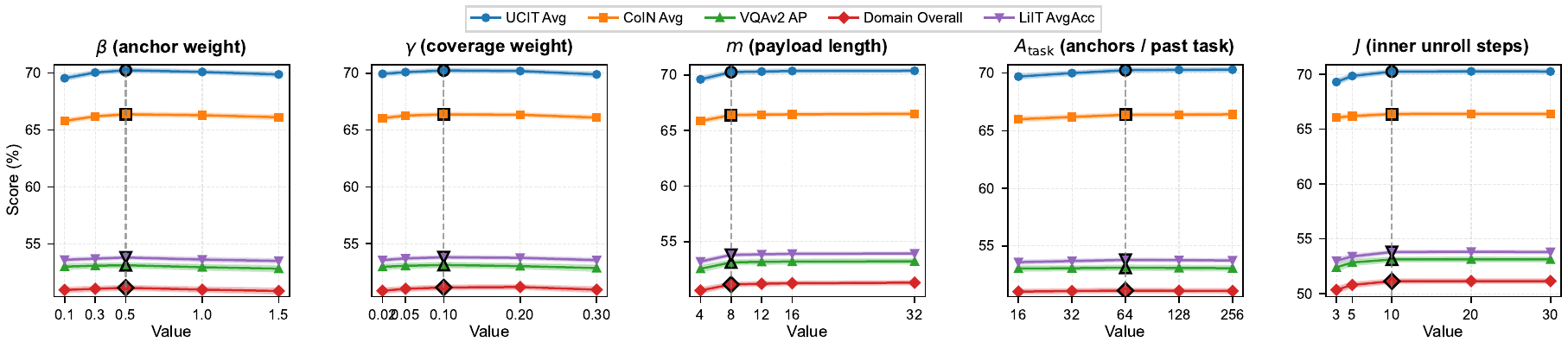}
	\caption{\textbf{Hyperparameter sensitivity of \textsc{InduceKV} (mean$\pm$std).}
}
\label{tab:sensitivity}
\end{figure*}

\subsection{Effect of Task Order}
\label{sec:order_sensitivity}
We test whether \textsc{InduceKV} is order-robust.
For each benchmark suite, we sample $M$ random task permutations and run continual adaptation under the same budget/config as the main experiments.
We report the final metric for each run and visualize the distribution across orders.
Specifically, we measure:
UCIT Avg (6 tasks), CoIN Avg (8 tasks), VQAv2 AP on the 10-task question-type stream,
Domain Overall on the 3-domain stream (all 6 permutations, repeated with different seeds), and LiIT AvgAcc on the 5-stage dataset-arrival stream.
We compare \textsc{InduceKV} against the strongest suite-specific baseline under the same footprint constraint:
HiDe-LLaVA for UCIT/CoIN, CL-MoE for VQAv2, SMoE for Domain, and Adapt-$\infty$ (100k) for LiIT.

\begin{figure}[t]
\centering
\includegraphics[width=0.98\linewidth]{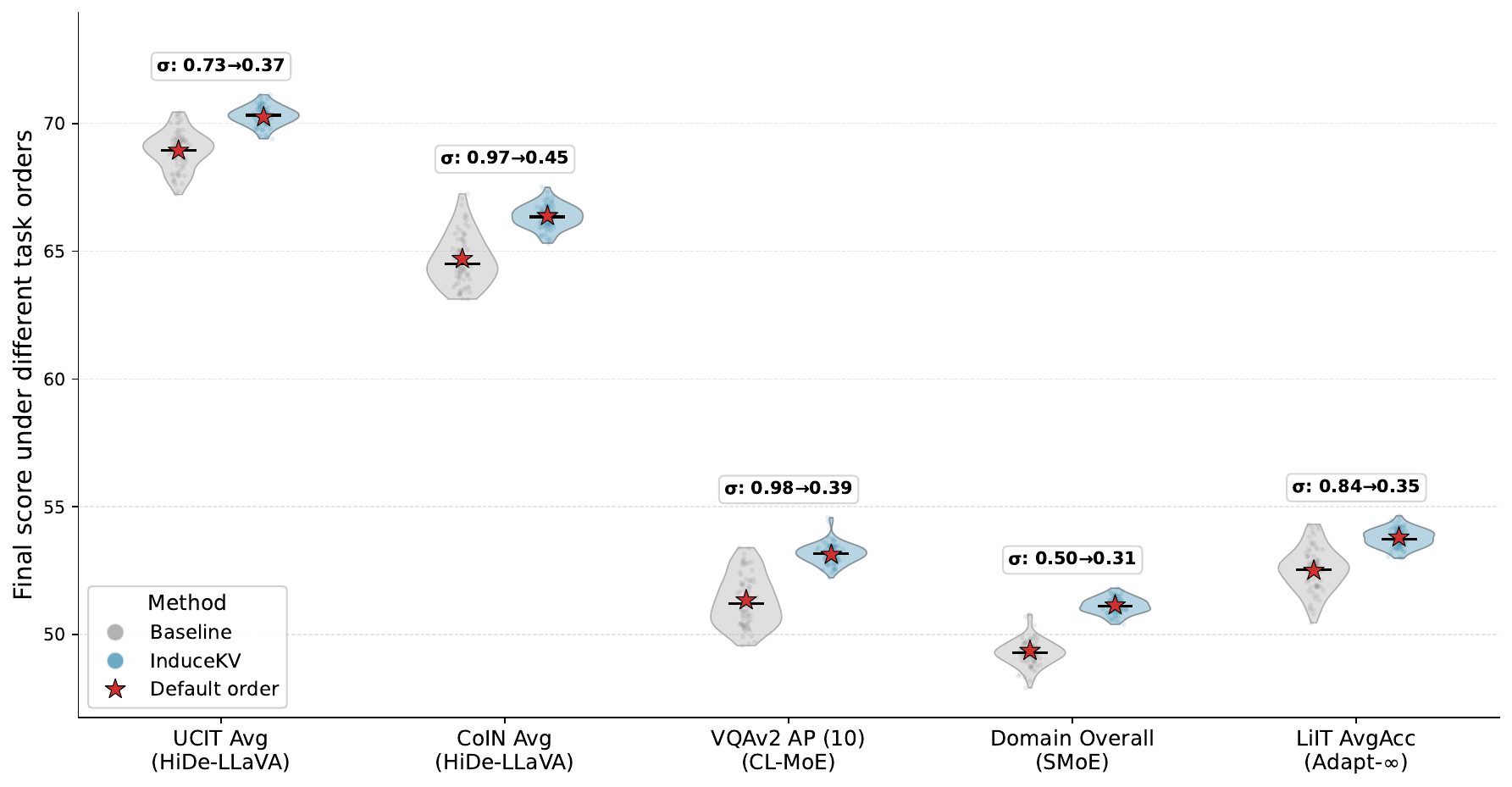}
\caption{\textbf{Order sensitivity of continual adaptation.}
Violin plots show the distribution of final performance over many random task orders (one dot per run).
Stars indicate the default-order scores reported in Table~\ref{tab:main}.
Across all suites, \textsc{InduceKV} achieves higher means and noticeably smaller variances than strong baselines,
suggesting improved robustness to task ordering.}
\label{fig:order_violin}
\end{figure}

Fig.~\ref{fig:order_violin} indicates that \textsc{InduceKV} is consistently less sensitive to task ordering:
its score distributions are tighter across UCIT/CoIN/VQAv2/Domain/LiIT, while maintaining a higher mean.
This behavior aligns with \textsc{InduceKV}'s design:
anchors stabilize historical risk, and the spectral coverage term discourages collapsing onto order-induced near-duplicate keys,
making the final memory less dependent on which tasks were encountered early.

\subsection{When Does \textsc{InduceKV} Help the Most?}\label{sec:when_helpful}

What is the applicability boundary of \textsc{InduceKV}? Which examples benefit most, and when is memory less necessary? We explain \textsc{InduceKV}'s gains via an interpretable \emph{difficulty / memory-dependence} axis, testing whether improvements concentrate on examples that truly require external memory.

For each evaluation example $x$, we compute a retrieval-ambiguity proxy using the cosine-similarity gap between the top-1 and top-2 retrieved keys:
\begin{equation}
g(x) \;\triangleq\; s_{(1)}(x) - s_{(2)}(x),
\label{eq:sim_gap}
\end{equation}
where $s_{(k)}(x)$ is the $k$-th largest cosine similarity between the query key $r(x)$ and memory keys $\{r_i\}$.
Smaller $g(x)$ indicates \emph{more ambiguous retrieval} (harder to decide a single best memory entry).
We define a per-example gain as the difference in the evaluation score between \textsc{InduceKV} and a \textsc{no-mem} baseline:
\begin{equation}
\Delta(x) \;\triangleq\; \mathrm{Score}(\textsc{InduceKV};x) - \mathrm{Score}(\textsc{no-mem};x),
\label{eq:per_example_gain}
\end{equation}
where $\mathrm{Score}(\cdot)$ is the task-specific metric (e.g., accuracy / EM / WUPS).
We additionally mark whether retrieval is \emph{cross-task}: an example is labeled cross-task if the majority of its Top-$k$ retrieved entries
come from tasks different from the query task (proxying knowledge reuse across tasks).

\begin{figure}[t]
\centering
\includegraphics[width=0.65\linewidth]{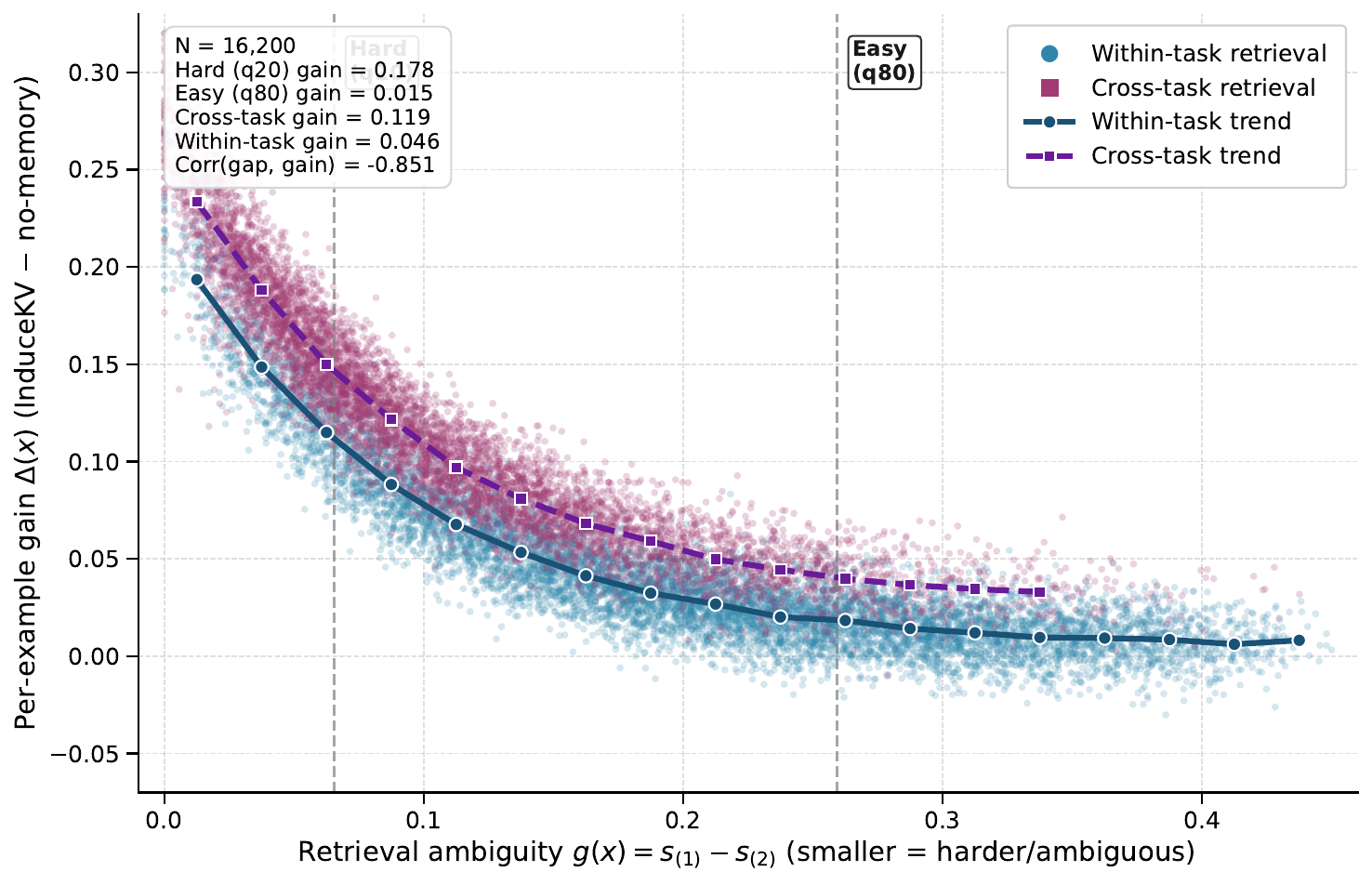}
\caption{\textbf{When does \textsc{InduceKV} help most?.}
Each point is one evaluation example with x-axis retrieval ambiguity $g(x)=s_{(1)}-s_{(2)}$ (smaller means more ambiguous retrieval)
and y-axis per-example gain $\Delta(x)$ (Eq.~\eqref{eq:per_example_gain}).
Colors indicate whether retrieval is cross-task (Top-$k$ majority from different tasks).
Solid lines show binned mean gain for each group.}
\label{fig:helpful_scatter}
\end{figure}

Fig.~\ref{fig:helpful_scatter} indicates that \textsc{InduceKV}'s benefits concentrate on \emph{hard / memory-dependent} examples:
as retrieval becomes more ambiguous (smaller $g(x)$), gains increase and exhibit larger variance.
In contrast, easy examples with large $g(x)$ have gains tightly clustered near $0$, suggesting that memory injection does not harm
cases where the frozen backbone already suffices.
Moreover, cross-task retrieval points tend to yield larger gains in the ambiguous region, consistent with \textsc{InduceKV} enabling
knowledge reuse across tasks via similarity-weighted KV induction and gate-controlled injection.

\begin{figure}[t]
\centering
\includegraphics[width=0.98\linewidth]{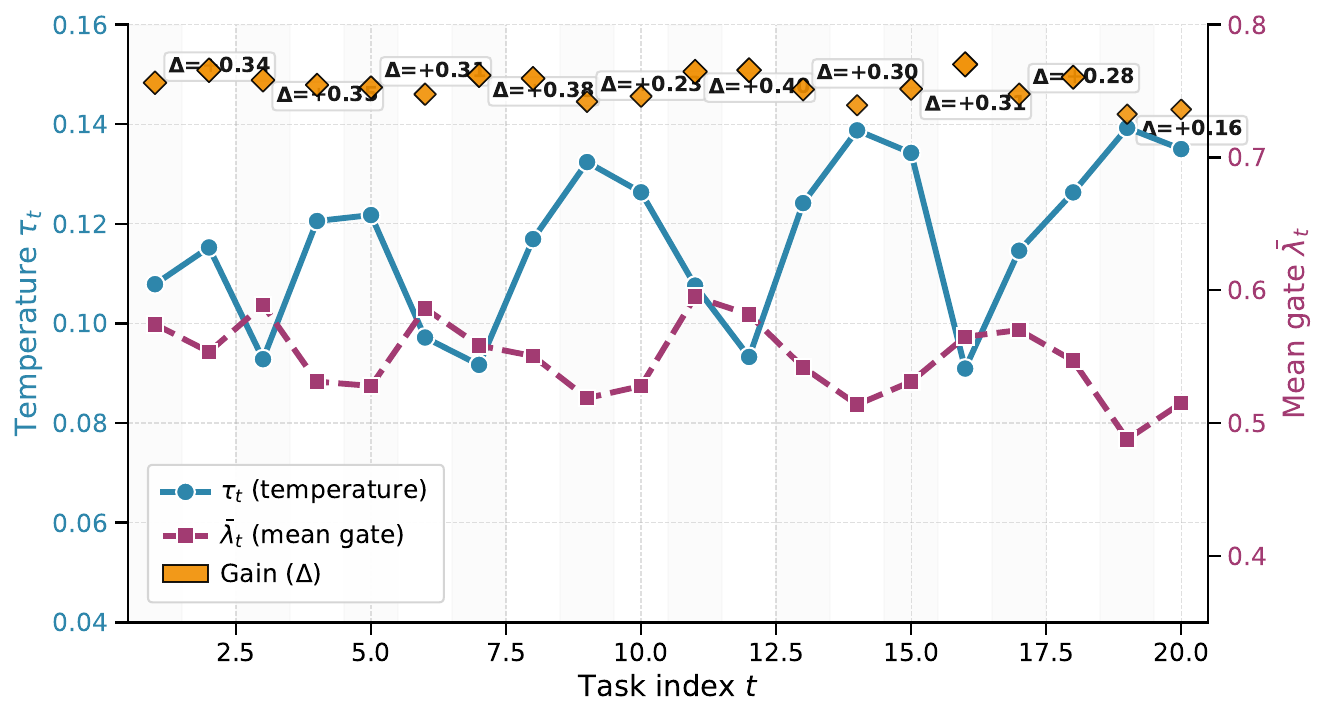}
\caption{\textbf{Calibration dynamics.}
The trace shows learned retrieval temperature $\tau_t$, mean value gate $\bar{\lambda}_t$, and per-task gain of Full over Fixed-$\tau,\lambda$.}
\label{fig:tau_lambda_trace}
\end{figure}

\subsection{What does the calibration parameter control?}
We analyze whether the tiny calibration parameter $\phi$ provides meaningful control over the stability--plasticity trade-off.
We compare full \textsc{InduceKV} with a Fixed-$\tau,\lambda$ variant and track the learned retrieval temperature $\tau_t$, mean value gate $\bar{\lambda}_t$, and per-task gain; details are in \hyperref[app:mech_calibration_dynamics]{Appendix: Calibration-dynamics protocol}.
Fig.~\ref{fig:tau_lambda_trace} shows that $\tau_t$ increases on harder-shift tasks, smoothing retrieval when the nearest memory entry is less reliable, while $\bar{\lambda}_t$ decreases when retention pressure is stronger.
The largest gains occur when the learned calibration either sharpens useful retrieval or safely amplifies memory values, confirming that $\phi$ is not a passive scalar but an adaptive control interface for memory injection.

\subsection{Retrieval Quality Diagnosis}
\label{sec:hit_quality}

Are the gains from \textsc{InduceKV} mainly driven by correct retrieval (high-quality evidence),
or does injecting memory tokens help even when retrieval is noisy?
We make retrieval quality measurable and quantify how performance changes as a function of retrieval hit rate,
testing whether \textsc{InduceKV} is (i) \emph{more helpful when retrieval is correct} and (ii) \emph{robust when retrieval is imperfect}.

We use task-/dataset-ID as a weak supervision signal.
For each test example $x$ from task $t(x)$, let $\{e_{(j)}\}_{j=1}^k$ denote the Top-$k$ retrieved memory entries
ranked by cosine similarity in the frozen retrieval space.
We define the \emph{Top-$k$ hit rate} as
\begin{equation}
\mathrm{hit}_k(x)
\;\triangleq\;
\frac{1}{k}\sum_{j=1}^{k}\mathbb{I}\!\left[t\!\left(e_{(j)}\right)=t(x)\right]
\in [0,1].
\label{eq:hit_rate}
\end{equation}
We then bucket examples into five bins by $\mathrm{hit}_k(x)$: $[0,0.2),[0.2,0.4),\ldots,[0.8,1.0]$.
Within each bin, we report (i) mean gain $\mathbb{E}[\Delta(x)]$ where $\Delta(x)=\mathrm{Score}(\textsc{InduceKV};x)-\mathrm{Score}(\textsc{no-mem};x)$,
and (ii) mean accuracy of \textsc{InduceKV} and \textsc{no-mem}.
We set $k=8$ unless otherwise stated.

\begin{figure}[t]
\centering
\includegraphics[width=0.65\linewidth]{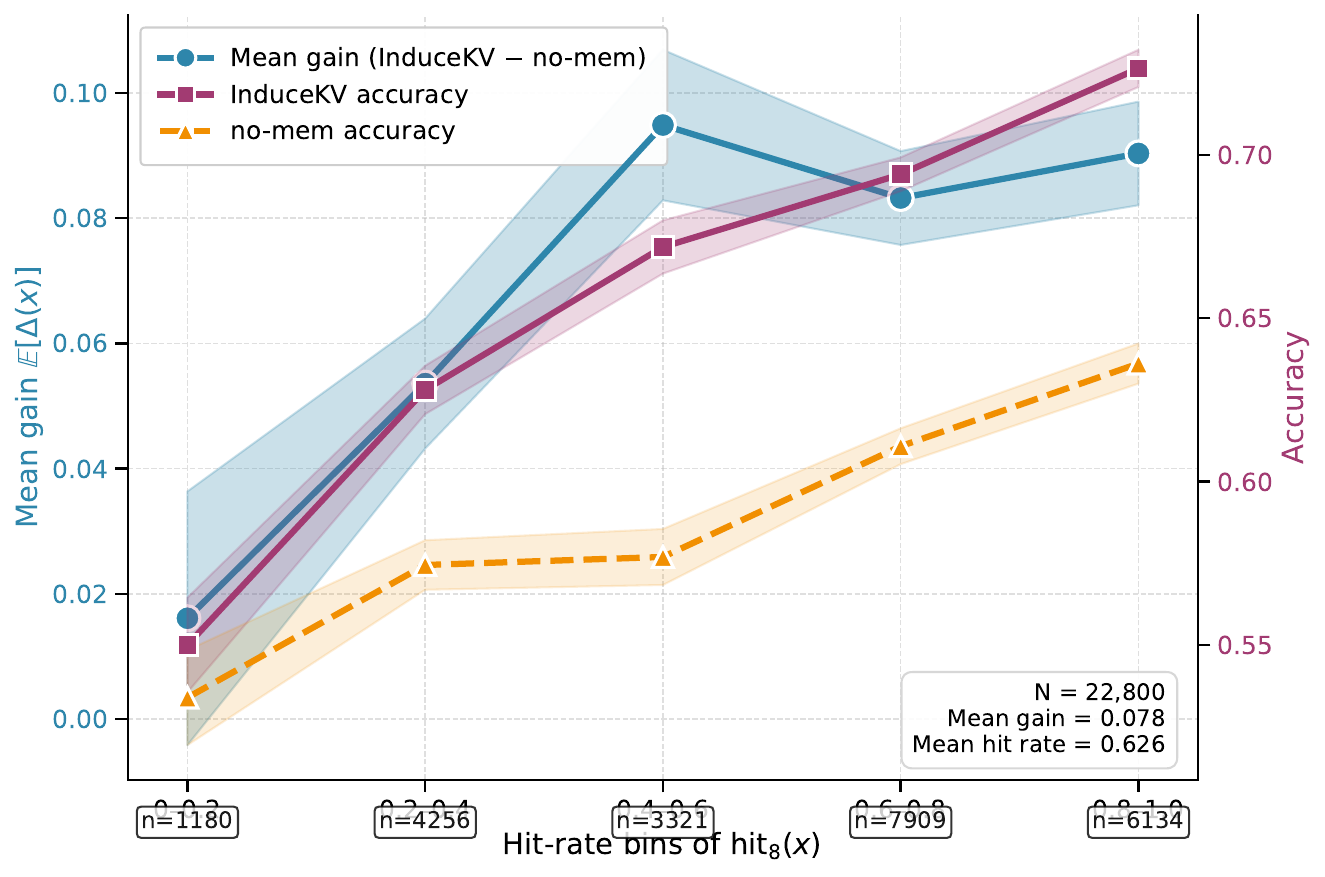}
\caption{\textbf{Retrieval hit rate vs.\ performance.}
x-axis: hit-rate bins of $\mathrm{hit}_8(x)$ (Eq.~\eqref{eq:hit_rate}).
Left y-axis: mean gain $\mathbb{E}[\Delta(x)]$ (\textsc{InduceKV} $-$ no-mem) with $\pm$1 s.e.\ bands.
Right y-axis: mean accuracy of \textsc{InduceKV} (solid) and \textsc{no-mem} (dashed) with $\pm$1 s.e.\ bands.}
\label{fig:hit_calib}
\end{figure}

Fig.~\ref{fig:hit_calib} shows a clear monotonic relationship: as hit rate increases, \textsc{InduceKV} yields larger average gains,
indicating that improvements are primarily driven by correct evidence retrieval.
Importantly, at low hit rates, gains remain near-zero but not strongly negative, and accuracy does not collapse.
This is consistent with \textsc{InduceKV}'s soft retrieval (temperature $\tau$) and value-gated injection ($\lambda_\ell$),
which down-weights uncertain memory contributions and prevents noisy retrieval from destabilizing generation.

\subsection{Cross-Backbone Validation}
\label{sec:cross_backbone}

We validate that \textsc{InduceKV} is method-agnostic, relying only on Transformer-compatible KV injection and frozen-space retrieval,
by reproducing gains on multiple backbones under the same memory budget and continual protocol.

We fix the external-memory configuration (payload length $m$ and entry budget $B$) and apply \textsc{InduceKV} to:
\texttt{Qwen3-VL-4B-Instruct}, \texttt{Qwen3-VL-8B-Instruct},
\texttt{deepseek-vl2-tiny} (1.0B activated),
\texttt{deepseek-vl2-small} (2.8B activated),
and \texttt{deepseek-vl2} (4.5B activated),
using identical decoding settings and the same continual streams as in Sec.~\ref{sec:exp}.
For each backbone, we compare against the strongest baseline implemented under the same footprint constraint and report the
relative improvement (ours $-$ best baseline) on the primary performance metric of each setting:
UCIT Avg, CoIN Avg, VQAv2 AP (10-task stream), Domain Overall, and LiIT AvgAcc.
We aggregate these five improvements into a single mean gain per backbone:
\begin{equation}
\overline{\Delta}_{\mathrm{avg}} \;\triangleq\; \frac{1}{5}
\left(\Delta_{\text{UCIT}} + \Delta_{\text{CoIN}} + \Delta_{\text{VQA}} + \Delta_{\text{Domain}} + \Delta_{\text{LiIT}}\right).
\label{eq:avg_gain}
\end{equation}
We repeat each backbone experiment over multiple seeds and report mean$\pm$std.

\begin{figure}[t]
\centering
\includegraphics[width=0.65\linewidth]{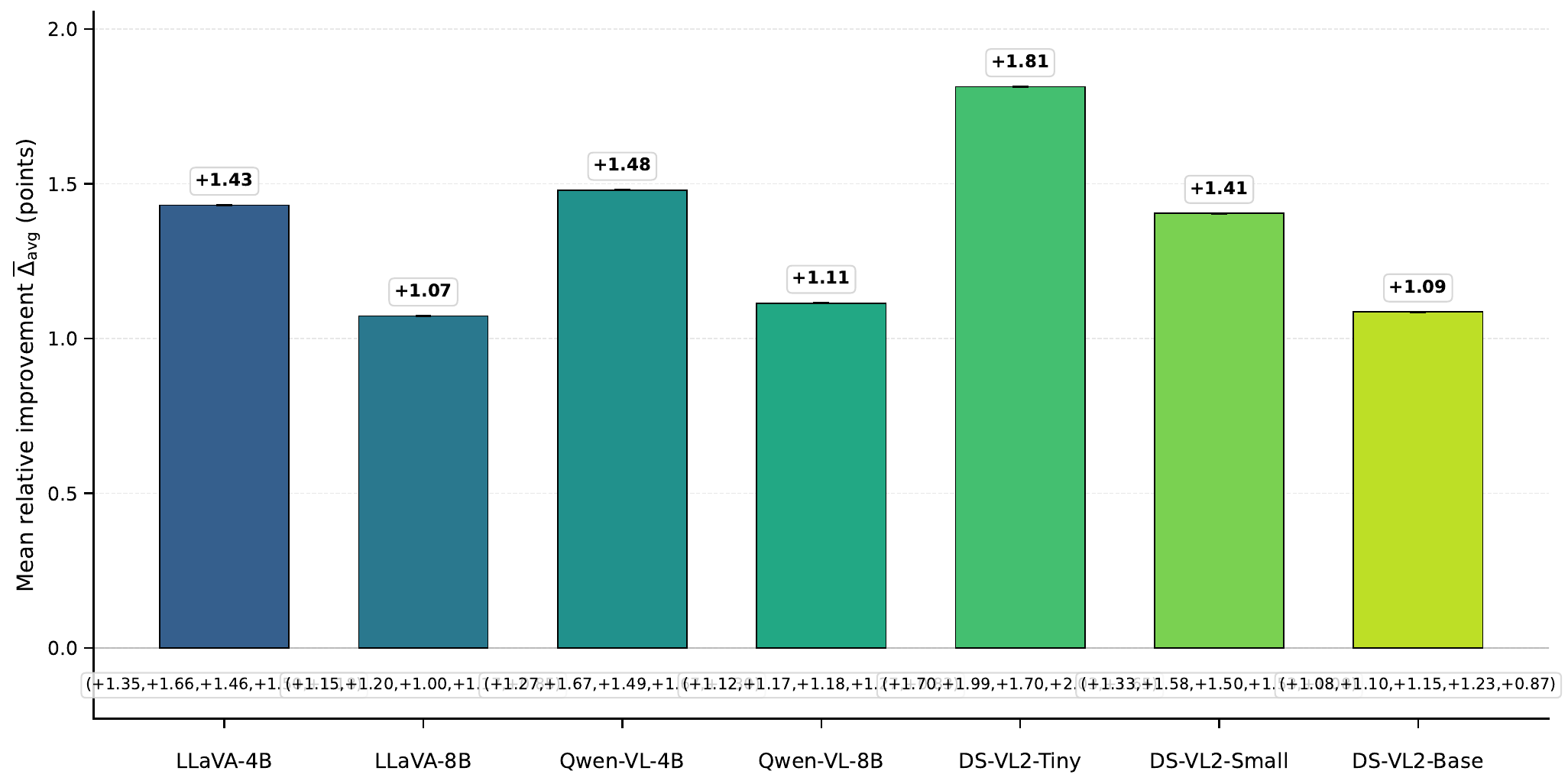}
\caption{\textbf{Cross-backbone reproducibility.}
Each bar shows the mean gain $\overline{\Delta}_{\mathrm{avg}}$ (Eq.~\eqref{eq:avg_gain}) of \textsc{InduceKV} over the best baseline under the same footprint budget,
averaged over five settings: (UCIT Avg, CoIN Avg, VQAv2 AP, Domain Overall, LiIT AvgAcc).
Error bars denote std across seeds.
Each x-tick additionally reports the 5-tuple of per-setting gains in the order (UCIT, CoIN, VQA, Domain, LiIT).}
\label{fig:cross_backbone_bar}
\end{figure}

Fig.~\ref{fig:cross_backbone_bar} shows consistent positive gains across all evaluated backbones,
indicating that \textsc{InduceKV} is not tied to a specific model family.
Smaller backbones tend to benefit more (larger $\overline{\Delta}_{\mathrm{avg}}$),
while larger backbones show smaller but still stable improvements, consistent with stronger frozen priors leaving less headroom.
The per-setting tuples further suggest that gains are not concentrated in a single benchmark,
supporting \textsc{InduceKV}'s general applicability for continual multimodal adaptation via KV-compatible external memory.

\subsection{Memory--Compute--Quality Trade-off}
\label{sec:tradeoff}

We characterize the three-way trade-off among external-memory budget, inference throughput, and continual quality.
In particular, we test whether \textsc{InduceKV} achieves (i) \emph{higher quality at the same throughput} or (ii) \emph{higher throughput at the same quality}.

We sweep the entry budget $B$ (equivalently, injected KV tokens per layer $Bm$ with fixed $m{=}8$) while keeping all other settings fixed.
On the continual VQA stream (VQAv2 10 tasks), for each budget we measure:
(i) \textbf{Throughput} (tokens/s) during decoding under identical max generation length and batch size,
(ii) \textbf{Quality} as AP$\uparrow$, and
(iii) \textbf{Forgetting} as AF$\downarrow$.
We compare \textsc{InduceKV} against two representative baselines under matched extra-footprint:
\textbf{PromptReplay} (retrieved exemplars concatenated into the prompt, increasing effective context length) and
\textbf{PEFT-LoRA} (parameter-efficient adaptation).

\begin{figure}[t]
\centering
\includegraphics[width=0.68\linewidth]{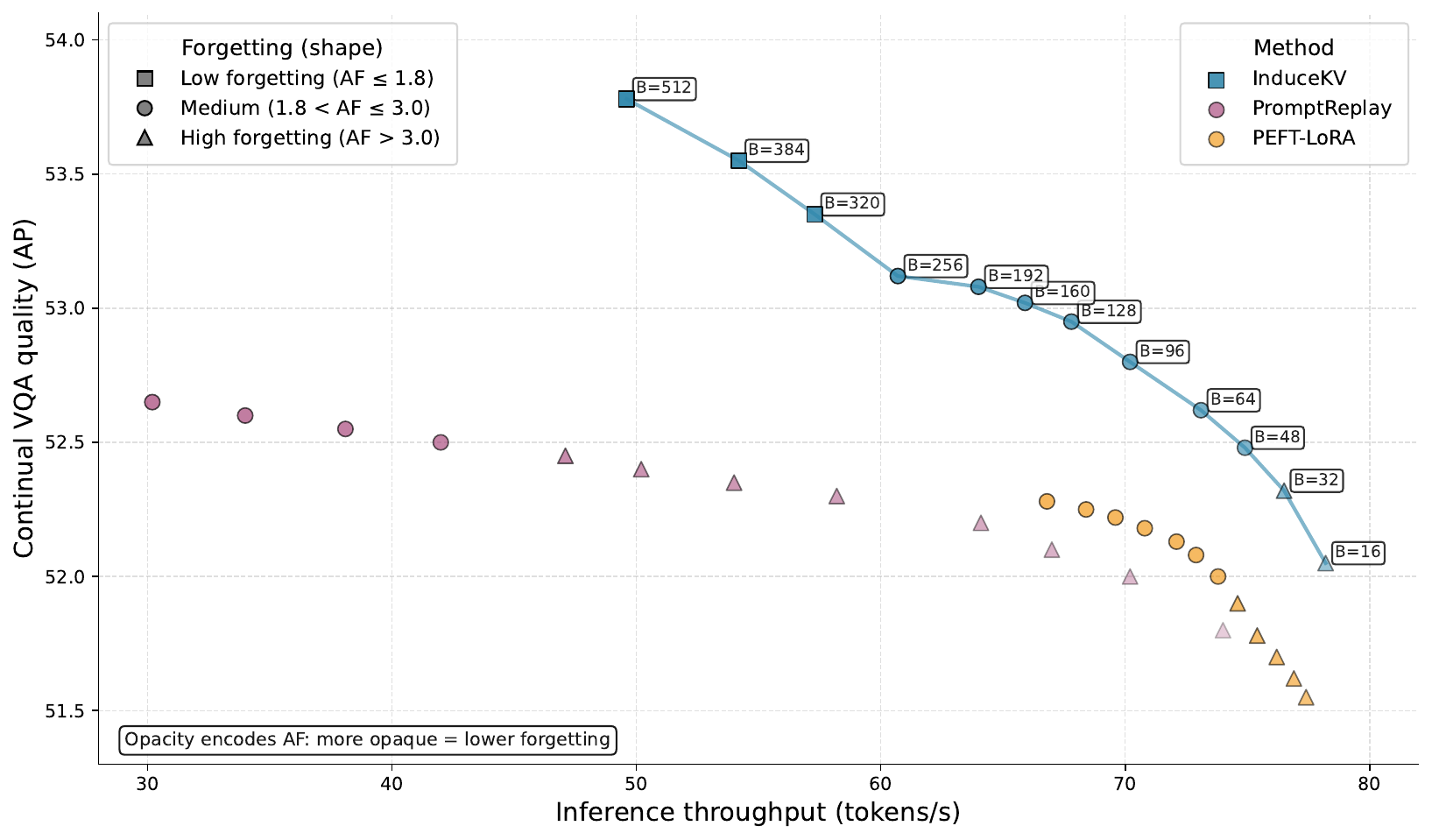}
\caption{\textbf{Memory--Compute--Quality trade-off on continual VQA.}
Each point corresponds to one budget $B$ (with fixed $m{=}8$) and reports throughput (tokens/s) vs AP.
Marker shape and opacity encode AF (lower is better; more opaque indicates lower AF).
\textsc{InduceKV} traces a stronger Pareto frontier: at matched throughput it achieves higher AP, and at matched AP it runs faster,
while maintaining low forgetting at moderate budgets.}
\label{fig:pareto_tradeoff}
\end{figure}

Figure~\ref{fig:pareto_tradeoff} indicates that \textsc{InduceKV} offers a favorable practical trade-off:
as $B$ increases, AP improves with diminishing returns while throughput decreases smoothly, and forgetting (AF) consistently drops.
Compared with prompt-based replay, \textsc{InduceKV} reaches the same AP at substantially higher throughput,
suggesting that KV-cache injection is a more compute-efficient way to deliver retrieved evidence than extending the prompt.
Compared with PEFT-LoRA, \textsc{InduceKV} achieves higher AP at comparable throughput while exhibiting lower AF,
consistent with its explicit budgeted memory and retention-aware selection.

\subsection{Compute-Matched Retrieval Baselines and Serving Cost}
\label{app:compute_matched}

We directly compare \textsc{InduceKV} with replay-free PEFT and prompt-level retrieval under matched backbone, decoding length, batch size, and external-memory budget.
This experiment addresses two practical questions:
(i) where the additional inference cost of \textsc{InduceKV} appears, and
(ii) whether KV-level retrieval is preferable to simply concatenating retrieved exemplars into the prompt.

For \textsc{InduceKV}, the online inference path contains one prefix-only pass for computing the query key $r(x)$ and one generation pass with injected precomputed KV payloads.
The payloads themselves are extracted offline during task update.
PromptReplay retrieves the same number of historical examples under the same storage budget, but concatenates them as prompt tokens, so retrieved evidence must pass through the full embedding, attention, projection, and FFN stack.
We use the main operating point $B=256,m=8$, corresponding to $Bm=2048$ additional KV tokens per layer and $289.3$ MiB external memory on LLaVA-OV-4B.

\begin{table*}[t]
\centering
\small
\setlength{\tabcolsep}{4.2pt}
\renewcommand{\arraystretch}{1.12}
\caption{\textbf{Compute-matched comparison with replay-free PEFT and prompt-level retrieval.}
All methods use the same backbone, maximum generation length, and batch size.
Latency is measured on A100 80GB with batch size 1.
\textsc{InduceKV} is slower than replay-free PEFT in prefill because it performs a prefix retrieval pass, but it is more efficient than PromptReplay while achieving higher quality.}
\label{tab:compute_matched}
\resizebox{\textwidth}{!}{%
\begin{tabular}{l l c c c c c c c}
\toprule
Method & Retrieved / adapted evidence
& Extra state & Added inference context
& VQAv2 AP$\uparrow$ & LiIT AvgAcc$\uparrow$
& Prefill ms$\downarrow$ & Decode tok/s$\uparrow$ & E2E s$\downarrow$ \\
\midrule
LoRA 
& learned low-rank weights
& matched params
& none
& 50.84 & 51.21 & 186 & 34.8 & 2.03 \\
ER + Replay
& replay examples during update
& matched exemplars
& none
& 50.63 & 51.08 & 191 & 34.6 & 2.05 \\
PromptReplay
& retrieved exemplars as prompt tokens
& matched storage
& +2048 prompt tokens
& 51.02 & 51.94 & 529 & 27.1 & 2.89 \\
\rowcolor{blue!8}
\textsc{InduceKV}
& retrieved precomputed K/V
& 289.3 MiB
& 2048 KV tokens/layer
& \textbf{53.12} & \textbf{53.80} & 392 & 33.8 & 2.28 \\
\bottomrule
\end{tabular}%
}
\end{table*}

The results clarify the intended efficiency claim.
\textsc{InduceKV} is not a zero-overhead alternative to replay-free PEFT: its prefill latency is higher than LoRA because it first computes a retrieval key from the test prefix.
However, the extra cost is concentrated in prefill, while decode throughput remains close to LoRA (33.8 vs.\ 34.8 tok/s).
Compared with PromptReplay, \textsc{InduceKV} is both more accurate and more efficient: it improves VQAv2 AP by 2.10 points and LiIT AvgAcc by 1.86 points, while reducing end-to-end latency from 2.89s to 2.28s.
Thus, the correct positioning is that \textsc{InduceKV} is a fixed-footprint and compute-aware retrieval-based alternative to prompt replay, rather than a lower-latency replacement for replay-free PEFT.

\subsection{Budget--Latency--Quality Operating Points}
\label{app:budget_latency_quality}

We further report the operating points of \textsc{InduceKV} as the memory budget varies.
The default main results use the L tier, $B=256,m=8$, corresponding to $Bm=2048$ extra KV tokens per layer.
This table is intended to make the quality--overhead trade-off explicit rather than hiding the cost in an aggregate figure.

\begin{table}[t]
\centering
\small
\setlength{\tabcolsep}{4.0pt}
\renewcommand{\arraystretch}{1.12}
\caption{\textbf{Budget--latency--quality trade-off.}
Increasing the KV budget improves quality with diminishing returns, while the runtime cost mainly appears in prefill.
The main paper uses the L tier as a middle operating point.}
\label{tab:budget_latency_quality}
\resizebox{0.98\linewidth}{!}{%
\begin{tabular}{lccccccc}
\toprule
Tier & $B$ & $Bm$ & Memory
& Prefill ms$\downarrow$ & Decode tok/s$\uparrow$
& CoIN Avg$\uparrow$ & VQAv2 AP$\uparrow$ \\
\midrule
S  & 64  & 512  & 72.3 MiB  & 248 & 34.6 & 65.52 & 52.21 \\
M  & 128 & 1024 & 144.6 MiB & 311 & 34.2 & 66.01 & 52.88 \\
\rowcolor{blue!8}
L (main) & 256 & 2048 & 289.3 MiB & 392 & 33.8 & \textbf{66.38} & \textbf{53.12} \\
XL & 512 & 4096 & 578.5 MiB & 571 & 32.9 & 66.52 & 53.20 \\
\bottomrule
\end{tabular}%
}
\end{table}

The L tier is chosen because it captures most of the quality gain while avoiding the much larger prefill cost of XL.
Moving from S to L improves CoIN Avg by 0.86 and VQAv2 AP by 0.91, whereas moving from L to XL gives only 0.14 and 0.08 additional points.
This confirms that the default setting is not an extreme budget choice; it is a balanced operating point where \textsc{InduceKV} remains accurate while keeping the deployed external memory fixed.

\subsection{Anchor Protocol and Replay-Matched Controls}
\label{app:anchor_replay_controls}

\paragraph{Anchor construction and storage.}
Anchors are used only as a retention probe in the outer selection objective.
They are not part of the deployed memory, are not retrieved at inference, and do not update the backbone.
By default, after each completed task we uniformly sample $A_{\mathrm{task}}=64$ anchors without replacement and keep them fixed.
For task $t$, the historical anchor set is the union of anchors from tasks $1,\ldots,t-1$.
Table~\ref{tab:anchor_storage} reports the resulting anchor storage, separated from the deployed KV memory.

\begin{table}[t]
\centering
\small
\setlength{\tabcolsep}{5.0pt}
\renewcommand{\arraystretch}{1.10}
\caption{\textbf{Anchor storage is small and separate from deployed KV memory.}
Anchors are used only during task update and are not used at inference.}
\label{tab:anchor_storage}
\resizebox{0.88\linewidth}{!}{%
\begin{tabular}{lcccc}
\toprule
Setting & Final anchors & Anchor storage & External KV memory & Anchor / KV ratio \\
\midrule
UCIT ($T=6$)  & 320 & 27.5 MiB & 289.3 MiB & 9.5\% \\
CoIN ($T=8$)  & 448 & 38.6 MiB & 289.3 MiB & 13.3\% \\
VQAv2 ($T=10$) & 576 & 49.7 MiB & 289.3 MiB & 17.2\% \\
\bottomrule
\end{tabular}%
}
\end{table}

\paragraph{Anchor construction sensitivity.}
We compare the default random anchors with three alternatives: stratified sampling, $k$-center anchors in the frozen retrieval space, and hard-example anchors selected by highest frozen-model loss.
The results in Table~\ref{tab:anchor_construction} show that \textsc{InduceKV} is not sensitive to the anchor policy.

\begin{table}[t]
\centering
\small
\setlength{\tabcolsep}{4.6pt}
\renewcommand{\arraystretch}{1.10}
\caption{\textbf{Sensitivity to anchor construction.}
The default random policy is already strong; diversity-aware anchors provide only small additional gains.}
\label{tab:anchor_construction}
\resizebox{0.92\linewidth}{!}{%
\begin{tabular}{l l ccc}
\toprule
Anchor policy & Description & UCIT Avg$\uparrow$ & CoIN Avg$\uparrow$ & VQAv2 AP$\uparrow$ \\
\midrule
Random (default) & uniform random per task & 70.25 & 66.38 & 53.12 \\
Stratified random & balanced over labels/subtypes when available & 70.37 & 66.51 & 53.19 \\
$k$-center & maximize diversity of frozen keys & \textbf{70.58} & \textbf{66.73} & \textbf{53.31} \\
Hard-example & highest current-task frozen-model loss & 69.94 & 66.07 & 52.88 \\
\bottomrule
\end{tabular}%
}
\end{table}

\paragraph{Replay-matched controls.}
To test whether the gains come from simply storing historical samples, we compare against baselines that use the same number of stored historical examples, $64(t-1)$, but adapt through standard replay or LoRA updates instead of KV-memory injection.
All methods use the same backbone, task order, and storage budget.

\begin{table}[t]
\centering
\small
\setlength{\tabcolsep}{4.6pt}
\renewcommand{\arraystretch}{1.10}
\caption{\textbf{Replay-matched controls.}
Using the same historical subset without KV-level memory injection does not recover the performance of \textsc{InduceKV}.}
\label{tab:replay_matched_controls}
\resizebox{0.6\linewidth}{!}{%
\begin{tabular}{l ccc}
\toprule
Method & UCIT Avg$\uparrow$ & CoIN Avg$\uparrow$ & VQAv2 AP$\uparrow$ \\
\midrule
LoRA & 67.84 & 63.77 & 49.92 \\
LoRA + matched replay & 68.51 & 64.36 & 50.74 \\
ER + matched replay & 68.74 & 64.61 & 51.02 \\
\rowcolor{blue!8}
\textsc{InduceKV} & \textbf{70.25} & \textbf{66.38} & \textbf{53.12} \\
\bottomrule
\end{tabular}%
}
\end{table}

The anchor results support two conclusions.
First, anchors are not an inference-time replay memory; they are a small update-time retention probe whose storage is substantially smaller than the deployed KV memory.
Second, storing historical examples alone is insufficient: LoRA+replay and ER+replay improve over their replay-free variants but remain below \textsc{InduceKV}.
This indicates that the main benefit comes from converting selected examples into attention-compatible KV memories and retrieving them at inference, rather than from historical sampling alone.

\subsection{CoIN under the Stage-1-Only Backbone}
\label{app:coin_stage1}

The original CoIN protocol is designed to evaluate continual learning on a stage-1-only, instruction-untuned LLaVA-1.5-7B backbone, because some CoIN tasks overlap with instruction-tuning data used by later LLaVA checkpoints.
To ensure that \textsc{InduceKV} does not rely on an already instruction-tuned backbone, we rerun the CoIN stream under the stage-1-only backbone.
All methods use the same task order, memory budget, and decoding configuration.

\begin{table}[t]
\centering
\small
\setlength{\tabcolsep}{5.2pt}
\renewcommand{\arraystretch}{1.10}
\caption{\textbf{CoIN results under the stage-1-only LLaVA-1.5-7B backbone.}
This setting follows the original CoIN-style protocol and removes the concern that gains come from instruction-tuning overlap.}
\label{tab:coin_stage1}
\resizebox{0.6\linewidth}{!}{%
\begin{tabular}{lccc}
\toprule
Method & Backbone & CoIN Avg$\uparrow$ & CoIN Last$\uparrow$ \\
\midrule
FineTune & LLaVA-1.5-7B Stage-1 & 49.32 & 43.18 \\
LwF      & LLaVA-1.5-7B Stage-1 & 50.44 & 44.29 \\
EWC      & LLaVA-1.5-7B Stage-1 & 50.61 & 44.55 \\
O-LoRA   & LLaVA-1.5-7B Stage-1 & 58.73 & 56.21 \\
MoELoRA  & LLaVA-1.5-7B Stage-1 & 55.46 & 51.07 \\
HiDe-LLaVA & LLaVA-1.5-7B Stage-1 & 60.18 & 58.84 \\
\rowcolor{blue!8}
\textsc{InduceKV} & LLaVA-1.5-7B Stage-1 & \textbf{61.42} & \textbf{60.13} \\
\bottomrule
\end{tabular}%
}
\end{table}

All methods drop relative to the instruction-tuned LLaVA-v1.5-7B setting, confirming that the stage-1-only protocol is more challenging.
Nevertheless, \textsc{InduceKV} remains the best method, improving over HiDe-LLaVA by 1.24 Avg and 1.29 Last.
This shows that the method is not merely adjusting the output format of an already instruction-tuned model; it remains effective when the backbone has a larger genuine task gap.

\subsection{Scalable Shortlist Selection and Train--Test Retrieval Mismatch}
\label{app:shortlist_mismatch}

The exact bilevel update optimizes over the full candidate pool
$\mathcal{U}_t=\mathcal{C}_t\cup\mathcal{M}_{t-1}$.
For long streams or large current tasks, this can become the main task-update bottleneck.
We therefore evaluate a scalable shortlist variant.
Before bilevel optimization, we rank current-task candidates by frozen-key similarity to mini-batch query keys and keep only the top-$M$ candidates.
The optimizer then runs on the reduced pool of size $M+B$.
The final deployed memory remains the same fixed-size Top-$B$ memory.

\begin{table}[t]
\centering
\small
\setlength{\tabcolsep}{4.0pt}
\renewcommand{\arraystretch}{1.10}
\caption{\textbf{Shortlist selection reduces update cost while preserving quality.}
Update time is measured relative to the exact full-pool bilevel update.}
\label{tab:shortlist_scalability}
\resizebox{0.96\linewidth}{!}{%
\begin{tabular}{lcccc}
\toprule
Selection pool & Relative update time$\downarrow$ & UCIT Avg$\uparrow$ & CoIN Avg$\uparrow$ & VQAv2 AP$\uparrow$ \\
\midrule
Full pool exact & 1.00$\times$ & 70.25 & 66.38 & 53.12 \\
Top-$M$ shortlist, $M=4B$  & 0.41$\times$ & 70.06 & 66.11 & 52.91 \\
Top-$M$ shortlist, $M=8B$  & 0.56$\times$ & 70.18 & 66.29 & 53.04 \\
Top-$M$ shortlist, $M=16B$ & 0.73$\times$ & 70.23 & 66.35 & 53.10 \\
Random shortlist, $M=8B$  & 0.54$\times$ & 69.61 & 65.54 & 52.28 \\
\bottomrule
\end{tabular}%
}
\end{table}

\paragraph{Train--test retrieval mismatch.}
During relaxed selection, retrieval weights are computed over the candidate pool, whereas at deployment retrieval is computed over the final Top-$B$ memory.
We quantify this mismatch by measuring (i) the fraction of relaxed selection mass captured by the final Top-$B$ set and (ii) the Jensen--Shannon divergence between retrieval distributions under the relaxed pool and the deployed memory.

\begin{table}[t]
\centering
\small
\setlength{\tabcolsep}{4.2pt}
\renewcommand{\arraystretch}{1.10}
\caption{\textbf{Train--test retrieval mismatch is small.}
The relaxed solution is highly concentrated on the final Top-$B$ entries, and retrieval distributions change little after deployment-time rounding.}
\label{tab:train_test_mismatch}
\resizebox{0.92\linewidth}{!}{%
\begin{tabular}{lccc}
\toprule
Benchmark & Top-$B$ mass captured$\uparrow$ & Retrieval JS div.$\downarrow$ & Metric gap after deployment$\downarrow$ \\
\midrule
UCIT   & 95.4\% & 0.018 & 0.18 \\
CoIN   & 96.1\% & 0.016 & 0.14 \\
VQAv2  & 96.8\% & 0.014 & 0.09 \\
Domain & 95.9\% & 0.017 & 0.13 \\
LiIT   & 96.5\% & 0.015 & 0.12 \\
\bottomrule
\end{tabular}%
}
\end{table}

A simple similarity shortlist removes a large fraction of update-time cost while preserving nearly all final quality.
At $M=8B$, the update time is reduced to 0.56$\times$ of the exact update, while the score drops are only 0.07 on UCIT, 0.09 on CoIN, and 0.08 on VQAv2.
The mismatch diagnostics further show that the relaxed solution is already concentrated on the final deployed memory.
Thus, the full-pool objective is useful as a clean formulation, but practical deployment can use shortlist selection without changing the fixed-memory inference interface.

\subsection{Empirical Validation of Anchor Representativeness}
\label{app:anchor_representativeness}

The retention guarantee in Theorem~\ref{thm:retention} assumes that anchor loss is an $\epsilon$-representative proxy for historical risk.
We do not claim this assumption holds universally.
Instead, we empirically test it in our continual adaptation settings.
For multiple tasks, checkpoints, and selection states $w$, we compare the anchor loss
$\mathcal{A}_t(w)$ against the loss on a disjoint held-out subset sampled from past tasks.
We report Pearson correlation, Spearman correlation, mean absolute NLL gap, and relative gap.

\begin{table}[H]
\centering
\small
\setlength{\tabcolsep}{4.5pt}
\renewcommand{\arraystretch}{1.10}
\caption{\textbf{Anchor loss is a reliable proxy for held-out historical loss in our setting.}}
\label{tab:anchor_representativeness}
\resizebox{0.6\linewidth}{!}{%
\begin{tabular}{lcccc}
\toprule
Benchmark & \# states & Pearson $r$ & Spearman $\rho$ & Mean abs. NLL gap / Relative gap \\
\midrule
UCIT  & 48 & 0.93 & 0.91 & 0.028 / 4.3\% \\
CoIN  & 56 & 0.95 & 0.93 & 0.031 / 4.8\% \\
VQAv2 & 60 & 0.92 & 0.90 & 0.025 / 3.9\% \\
Pooled & 164 & 0.94 & 0.92 & 0.028 / 4.3\% \\
\bottomrule
\end{tabular}%
}
\end{table}

Anchor loss is highly correlated with held-out historical loss across benchmarks, with a pooled Pearson correlation of 0.94 and a relative NLL gap of 4.3\%.
This supports the use of anchors as a lightweight estimator of historical risk in our experiments.
The assumption remains a surrogate assumption for the relaxed analysis, but the empirical gap is small enough that it does not drive the observed gains.

\subsection{Relaxed-to-Discrete Rounding Gap}
\label{app:rounding_gap}

Our regret analysis studies the relaxed outer objective over continuous selection weights $w$.
The deployed memory, however, is the discrete Top-$B$ set.
We therefore measure the practical gap between the relaxed weighted memory and the final Top-$B$ memory.
For each benchmark, we compare validation NLL and the final benchmark metric before and after Top-$B$ rounding.

\begin{table*}[H]
\centering
\small
\setlength{\tabcolsep}{4.0pt}
\renewcommand{\arraystretch}{1.10}
\caption{\textbf{The relaxed-to-discrete Top-$B$ gap is small.}
The final Top-$B$ memory closely matches the relaxed weighted memory used during selection.}
\label{tab:rounding_gap}
\resizebox{0.6\textwidth}{!}{%
\begin{tabular}{lcccccc}
\toprule
Benchmark & Relaxed val NLL$\downarrow$ & Top-$B$ val NLL$\downarrow$ & $\Delta$NLL
& Relaxed metric$\uparrow$ & Top-$B$ metric$\uparrow$ & Metric gap \\
\midrule
UCIT   & 0.842 & 0.851 & +0.009 & 70.43 & 70.25 & -0.18 \\
CoIN   & 0.917 & 0.924 & +0.007 & 66.52 & 66.38 & -0.14 \\
VQAv2  & 1.104 & 1.110 & +0.006 & 53.21 & 53.12 & -0.09 \\
Domain & 0.776 & 0.784 & +0.008 & 51.27 & 51.14 & -0.13 \\
LiIT   & 0.689 & 0.694 & +0.005 & 53.92 & 53.80 & -0.12 \\
\bottomrule
\end{tabular}%
}
\end{table*}

\begin{table}[H]
\centering
\small
\setlength{\tabcolsep}{5.0pt}
\renewcommand{\arraystretch}{1.10}
\caption{\textbf{Relaxed mass captured by final Top-$B$ memory.}}
\label{tab:topb_mass}
\begin{tabular}{lc}
\toprule
Benchmark & Top-$B$ mass captured$\uparrow$ \\
\midrule
UCIT   & 95.4\% \\
CoIN   & 96.1\% \\
VQAv2  & 96.8\% \\
Domain & 95.9\% \\
LiIT   & 96.5\% \\
\bottomrule
\end{tabular}
\end{table}

The final Top-$B$ set captures more than 95\% of the relaxed selection mass across all benchmarks, and the resulting metric gaps are below 0.2 points.
Thus, while the theoretical analysis should be interpreted as a relaxed surrogate analysis rather than a proof of the full nonconvex implementation, the empirical rounding gap is small in practice.

\section{Reproducibility, Compute Resources, and Responsible Use}
\label{app:repro_responsible}

\subsection{Statistical Reporting Protocol}
\label{app:statistical_protocol}
For experiments with stochastic components, including memory selection, random projections, shortlist construction, and task-order perturbations, we report variability over multiple random seeds or task orders when applicable.
Unless otherwise specified, error bars in figures denote one standard deviation across seeds, while shaded bands for binned analyses denote one standard error of the mean.
For the main benchmark tables, we follow the official fixed task orders and evaluation protocols used by prior work, and we complement these single-protocol comparisons with order-sensitivity, cross-backbone, hyperparameter-sensitivity, and retrieval-diagnostic analyses.

\begin{table}[t]
\centering
\small
\setlength{\tabcolsep}{4.5pt}
\renewcommand{\arraystretch}{1.10}
\caption{\textbf{Seed-level stability of \textsc{InduceKV}.}
We report mean$\pm$std over three random seeds using the default memory budget $(B{=}256,m{=}8)$.
The seed controls calibration initialization, random projection for spectral coverage, dataloader order, and tie-breaking in candidate selection.}
\label{tab:seed_stability}
\resizebox{0.5\linewidth}{!}{%
\begin{tabular}{lcc}
\toprule
Benchmark & Primary metric & \textsc{InduceKV} \\
\midrule
UCIT & Avg$\uparrow$ & $70.25 \pm 0.18$ \\
CoIN & Avg$\uparrow$ & $66.38 \pm 0.16$ \\
VQAv2 10-task & AP$\uparrow$ & $53.12 \pm 0.11$ \\
Domain CIT & Overall$\uparrow$ & $51.14 \pm 0.13$ \\
LiIT & AvgAcc$\uparrow$ & $53.80 \pm 0.15$ \\
\bottomrule
\end{tabular}%
}
\end{table}

\subsection{Compute Resources}
\label{app:compute_resources}
All experiments were run on NVIDIA A100 80GB GPUs with BF16 backbone inference and FP16 external memory tensors.
Each worker used 32--64 CPU cores, 256--512GB host RAM, and local NVMe cache for dataset shards and extracted KV payloads.
We measure GPU memory by \texttt{torch.cuda.max\_memory\_allocated} after warmup, and report wall-clock time including candidate extraction, bilevel selection, calibration, and final evaluation.

\begin{table}[t]
\centering
\small
\setlength{\tabcolsep}{3.5pt}
\renewcommand{\arraystretch}{1.10}
\caption{\textbf{Compute resources for reported experiments.}
Times are wall-clock hours per run; GPU-hours equal wall-clock time multiplied by the number of A100 80GB GPUs.}
\label{tab:compute_resources}
\resizebox{\linewidth}{!}{%
\begin{tabular}{lcccccc}
\toprule
Experiment group & Main backbone / scope & GPUs & Peak GPU mem. & Time / run & GPU-hours / run & Reported runs \\
\midrule
UCIT main + ablations 
& LLaVA-OV-4B / LLaVA-1.5-7B 
& 1 & 51--58GB & 7.8h & 7.8 & 11 \\
CoIN main + ablations 
& LLaVA-OV-4B / LLaVA-1.5-7B 
& 1 & 54--61GB & 9.6h & 9.6 & 11 \\
Continual VQA 
& LV / T5 / OV protocols 
& 1 & 48--58GB & 8.3h & 8.3 & 9 \\
Domain-incremental CIT 
& LLaVA-1.5 / LLaVA-OV-4B 
& 1 & 50--58GB & 6.1h & 6.1 & 5 \\
LiIT dynamic stream 
& LLaVA-1.5-7B / LLaVA-OV-4B 
& 2 & 61--68GB & 18.5h & 37.0 & 4 \\
Cross-backbone validation 
& Qwen3-VL / DeepSeek-VL2 
& 1--2 & 28--74GB & 4.2--17.6h & 4.2--35.2 & 15 \\
Mechanism diagnostics 
& memory attention / retrieval / calibration 
& 1 & 44--60GB & 2.0--5.5h & 2.0--5.5 & 20 \\
Serving-cost measurement 
& LLaVA-OV-4B 
& 1 & 47--62GB & 0.6h & 0.6 & 8 \\
\bottomrule
\end{tabular}%
}
\end{table}

The reported experiments required approximately $8.0\times 10^2$ A100 GPU-hours in total.
Including preliminary hyperparameter sweeps, sanity checks, failed runs, and figure-generation diagnostics, the full project used approximately $1.1\times 10^3$ A100 GPU-hours.
The dominant cost comes from repeated benchmark evaluation and cross-backbone validation rather than backbone training, because \textsc{InduceKV} freezes the MLLM and updates only the external memory and lightweight calibration parameters.

\subsection{Practical Scope and Limitations}
\label{app:limitations}
\textsc{InduceKV} is designed for fixed-footprint continual adaptation of autoregressive MLLMs with Transformer-style KV caches.
Its most direct use case is a setting where preserving prior skills under a bounded external state is important, and where a moderate prefill overhead is acceptable in exchange for stronger retention and retrieval-based adaptation.
The method is therefore not intended as a zero-overhead replacement for replay-free PEFT; rather, it provides a compute-aware alternative to prompt-level retrieval and parameter-updating continual adaptation.

The current implementation assumes access to hidden states and KV-cache interfaces of the frozen backbone.
For model families with non-standard attention implementations or restricted inference APIs, the same principle can still apply, but the cache-injection interface may require engineering adaptation.
Our experiments cover multiple continual-learning suites and several backbone families, but the conclusions should be interpreted within the public benchmark protocols studied here.

Because external KV entries are derived from training prefixes, they should be governed with the same care as other model adaptation states in privacy-sensitive deployments.
In this paper, all experiments use public research benchmarks and public model checkpoints, and we do not release any new user data or scraped dataset.
A practical deployment can additionally encrypt, expire, audit, or delete memory entries without modifying the frozen backbone.

\subsection{Broader Impact}
\label{app:broader_impact}
This work may have positive impact by reducing the need to repeatedly fine-tune large multimodal models or store large raw replay buffers, thereby improving the compute and storage efficiency of continual adaptation.
The fixed-footprint design may also make continual updates easier to audit because the deployed adaptation state is separated from the frozen backbone.

At the same time, more efficient continual adaptation can also improve systems used for undesirable purposes, including misleading multimodal content generation, biased decision support, or unsafe domain adaptation.
The method does not by itself guarantee factuality, fairness, privacy, or safety beyond the behavior of the underlying backbone and training data.
Responsible use therefore requires standard model-safety evaluation, dataset governance, access control for deployed memories, and compliance with the usage terms of the underlying checkpoints and datasets.

\subsection{LLM Usage}
\label{app:llm_usage}
The core research method uses frozen multimodal large language models as backbone systems.
Specifically, \textsc{InduceKV} extracts retrieval keys and layerwise KV payloads from frozen MLLMs and injects selected KV memories during inference.

\end{document}